\documentclass[11pt]{article}

\usepackage[final]{acl}

\usepackage{times}
\usepackage{latexsym}

\usepackage[T1]{fontenc}

\usepackage[utf8]{inputenc}

\usepackage{microtype}

\usepackage{inconsolata}

\usepackage{graphicx}
\usepackage{tikz}
\usetikzlibrary{arrows.meta,calc,positioning}

\usepackage{enumitem}

\usepackage{amsmath}
\usepackage{amssymb}
\usepackage{amsthm}
\newtheorem{lemma}{Lemma}

\usepackage{booktabs}
\usepackage{multirow}
\usepackage{placeins}
\usepackage{colortbl}
\definecolor{conduithl}{gray}{0.92}

\definecolor{ConduitRed}{RGB}{190,48,48}
\definecolor{ConduitBlue}{RGB}{54,104,178}
\definecolor{ConduitGreen}{RGB}{66,145,88}
\definecolor{ConduitGray}{RGB}{105,105,105}
\definecolor{ConduitLight}{RGB}{238,238,238}
\definecolor{ConduitPanelBlue}{RGB}{145,203,254}

\usepackage{algorithm}
\usepackage{algpseudocode}

\definecolor{ProcBlue}{RGB}{0,92,175}

\algrenewcommand\algorithmicrequire{\textbf{Input:}}
\algrenewcommand\algorithmicensure{\textbf{Output:}}
\algrenewcommand\algorithmicreturn{\textbf{return}}

\newcommand{\stdpm}[1]{\textnormal{\scriptsize\,\textpm{}#1}}

\title{\textsc{Conduit}: A Unified Residual-Stream Restoration Framework for \\ KV Cache Reuse in Vision-Language Models}

\author{
  \textbf{Pengan Chen\textsuperscript{1,*}},
  \textbf{Kaisheng Zheng\textsuperscript{1,*}},
  \textbf{Liang Hong\textsuperscript{1}},
  \textbf{Lixia Yi\textsuperscript{2}},
  \textbf{Jiyue Jiang\textsuperscript{1}},
\\
  \textbf{Jiayang Chen\textsuperscript{1}},
  \textbf{Yixuan Wang\textsuperscript{1}},
  \textbf{Yimin Fan\textsuperscript{1}},
  \textbf{Xinyuan Liu\textsuperscript{1}},
\\
  \textbf{Jiayi Li\textsuperscript{1}},
  \textbf{Zhanqiu Zhang\textsuperscript{3,\textdagger}},
  \textbf{Yiwen Guo\textsuperscript{4,\textdagger}},
  \textbf{Yu Li\textsuperscript{1,\textdagger}}
\\
  \textsuperscript{1}The Chinese University of Hong Kong,
  \textsuperscript{2}Fudan University
\\
  \textsuperscript{3}LIGHTSPEED,
  \textsuperscript{4}Independent Researcher
}

\begin{document}
\maketitle
\begingroup
\renewcommand{\thefootnote}{\fnsymbol{footnote}}
\footnotetext[1]{Equal contribution.}
\footnotetext[2]{Co-corresponding authors.}
\endgroup

\begin{abstract}
Vision--language models (VLMs) often answer new questions about recurring visual content, where reusing the key--value (KV) cache can avoid re-encoding expensive visual prefixes. Exact-prefix reuse, however, fails when the same visual content appears under a changed prefix. Selective recomputation can recover quality under a small visual-token budget, but only when the right stale tokens are refreshed. Raw-attention selection can waste budget on high-attention tokens with small value-norm proxy scores and on query-irrelevant images. To address these failure modes, we propose \textsc{Conduit}, a training-free refresh policy that unifies single- and multi-image reuse as \emph{residual-stream restoration}. Building on norm-weighted attention, \textsc{Conduit} ranks cached visual tokens using cached-key query attention and an accessible pre-output cached-value-norm proxy, then applies empirical image-level relevance amplification before one global selection. With one image, the coefficient is one and the rule reduces to intra-image token selection. The method preserves model architecture and weights, adding only a single query-conditioned scoring pass at inference. At a $10\%$ refresh budget, \textsc{Conduit} achieves $97.0$--$99.5\%$ of the corresponding full-prefill five-dataset average across three VLM backbones and leads budgeted methods on average; on the MMLongBench-Doc latency subset, it uses $13.5\%$ of full-prefill FLOPs and achieves a $2.99\times$ time-to-first-token speedup.
\end{abstract}

\section{Introduction}
\label{sec:intro}

Vision--language models (VLMs) answer questions grounded in natural images, document pages, and user interfaces. In deployment, the same visual content often supports more than one query: retrieval-augmented generation over documents and screenshots \citep{yu2025visrag,tanaka2025vdocrag}, multi-turn visual dialogue \citep{das2017visualdialog}, multi-image reasoning \citep{jiang2024mantis,li2024llavanextinterleave}, and agentic loops \citep{koh2024visualwebarena,xie2024osworld} all revisit pages, slides, or screens as new questions arrive. Patch-based image encoders expand images into token sequences \citep{dosovitskiy2021vit}, and modern high-resolution VLM inputs can exceed a thousand visual tokens per image \citep{liu2024llavaonevision}; multi-image prompts can span tens of thousands of visual tokens, making repeated prefill expensive. Text-serving systems avoid this cost by caching the key--value (KV) tensors of an exact prefix and replaying them for a later request that shares it \citep{kwon2023pagedattention,zheng2024sglang}. Visual content, however, can recur under a changed prefix, leaving its cached state stale in the new causal context. We isolate this challenge with a two-request shifted-prefix protocol: the first request materializes visual K/V, and the second reuses the same visual content under a changed query prefix. Selective recomputation then keeps most cached entries and refreshes only a small subset under a fixed budget \citep{yao2024cacheblend}. The central question is not how much to refresh, but \emph{which} stale visual tokens to repair.

\begin{figure*}[t]
    \centering
    \includegraphics[width=\textwidth]{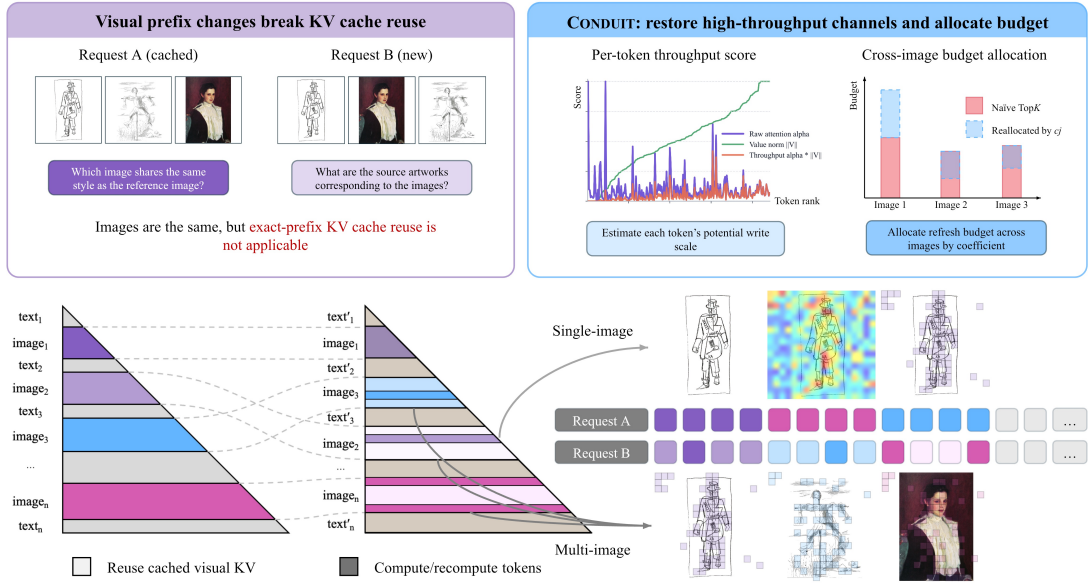}
    \caption{\textbf{Overview of \textsc{Conduit}.} Visual-prefix changes make exact-prefix KV reuse invalid even when image content recurs. \textsc{Conduit} builds a norm-weighted token score from cached-key query attention and an accessible pre-output cached-value-norm proxy, applies empirical image-level relevance amplification, and recomputes only the highest-scoring visual tokens while reusing the rest.}
    \label{fig:teaser}
\end{figure*}

Raw attention is a natural selection signal, yet it is incomplete for visual cache repair. Prior work shows that high attention need not imply a large transformed contribution to the residual stream \citep{kobayashi2020attention,xiao2023streamingllm,kang2024visualattentionsink,bondarenko2023quantizable,sun2024massive}. Among strongly read tokens, cached value norm provides an accessible proxy for relative write scale. Multi-image prompts add an allocation problem: attention is normalized over a mixed visual prefix, so salient but question-irrelevant patches from distractor images can enter the global high-attention set. A raw-attention top-$k$ selector can therefore spend refresh budget on tokens that do not support the answer, leaving fewer slots for the image that contains answer-bearing evidence. Under shifted context, the cached visual states also do not reflect the current cross-image interactions \citep{srikrishnan2025blindsight,zhao2024mirb}. We therefore combine norm-aware token ranking with empirical image-level relevance amplification before the shared budget is spent.

To instantiate this formulation, we propose \textsc{Conduit}. We view transformer hidden states as a \emph{residual stream}: a shared representation channel that attention layers read from and write to \citep{elhage2021transformer}. Selective recomputation then becomes \emph{residual-stream restoration}---we refresh cached visual tokens whose stale states are most likely to perturb this channel under the current query. This view separates two quantities that raw attention conflates: how strongly the query reads a token and the scale at which that token may write. Building on the norm-weighted attention analysis of \citet{kobayashi2020attention}, \textsc{Conduit} combines cached-key query attention with cached value norm as an accessible pre-output write-scale proxy. It then applies an empirical per-image coefficient that amplifies scores for images read more strongly by the current query before one global top-$k$ selection. With one image, the coefficient is one and the rule reduces to intra-image token selection. Thus, single-image cache reuse and multi-image budget reallocation become two instances of the same refresh principle.

Operationally, \textsc{Conduit} is a training-free refresh policy that leaves model architecture and weights unchanged. At inference, one query-conditioned scoring pass composes the norm-aware token score with its image coefficient and refreshes the highest-scoring visual tokens (Figure~\ref{fig:teaser}). This design demotes high-attention tokens with small value-norm proxy scores within an image and reallocates refresh pressure toward query-relevant images across a multi-image prompt. Across single- and multi-image settings, \textsc{Conduit} preserves most of full-prefill quality with substantially less prefill compute. Our contributions are:

\begin{itemize}[leftmargin=*,topsep=0.25em,itemsep=0.15em,parsep=0pt,partopsep=0pt]
    \item \textbf{A unifying formulation} that casts visual KV-cache reuse as residual-stream restoration, combining within-image norm-aware ranking and cross-image budget reweighting, with single-image reuse as the one-image limit.
    \item \textbf{\textsc{Conduit}}, a training-free, architecture-preserving refresh policy that operationalizes norm-weighted attention for stale visual caches and adds empirical image-level relevance amplification before one global selection, with one query-conditioned scoring pass and no learned parameters or weight changes.
    \item \textbf{An empirical evaluation} across three backbones (Qwen2.5-VL-3B, Qwen2.5-VL-7B, and InternVL3-9B), five datasets in the main shifted-prefix evaluation (LongDocURL, MMLongBench-Doc, SlideVQA, InfoSeek, and ViQuAE), and four $K{=}1$ benchmarks (MMBench, OCRBench, POPE, and MMStar). At a $10\%$ refresh budget, \textsc{Conduit} achieves $97.0$--$99.5\%$ of the corresponding full-prefill five-dataset average and leads budgeted baselines on average; on the MMLongBench-Doc latency subset, it uses $13.5\%$ of full-prefill FLOPs and achieves a $2.99\times$ time-to-first-token speedup. Factor ablations show that both value-norm weighting and image-level amplification contribute empirical gains.
\end{itemize}

\section{Related Work}
\label{sec:related}

\subsection{KV Cache Compression}
\label{sec:related:kv}

The first wave of KV cache work tackled the memory footprint of long text contexts by retaining only a fraction of past entries. \citet{zhang2023h2o} identify a small set of \emph{heavy-hitter} tokens whose accumulated attention dominates the rest and retain them while evicting the tail; \citet{liu2023scissorhands} refine the scoring at decode time by predicting which tokens will continue to receive attention. \citet{li2024snapkv} use a recent observation window to identify important KV positions and pool nearby entries, and \citet{xiao2023streamingllm} additionally retain the first few positions because eviction at attention sinks destabilizes generation. Subsequent work loosens the assumption that every layer or head needs the same budget through pyramidal \citep{cai2024pyramidkv}, head-adaptive \citep{feng2024adakv}, retrieval-vs-streaming \citep{xiao2024duoattention}, and online-buffer \citep{oren2024tova} variants, with further refinements on scoring \citep{jo2025fastkv,tian2025keepkv} and quantization \citep{liu2024kivi}. Across this line, attention magnitude remains a common importance signal, and the cache is typically treated as one uniform sequence. Visual prefixes complicate both assumptions: raw attention alone does not expose the scale of a token's transformed contribution \citep{kobayashi2020attention}, and the prefix is segmented into images that compete for one shared refresh budget.

\subsection{Query-Driven Cache Reuse}
\label{sec:related:rag}

A second line targets reuse of a precomputed cache when the query is known but the cached context has shifted, motivated by retrieval-augmented generation. \citet{yao2024cacheblend} show that partial recomputation over cached chunks recovers answer quality at a fraction of full-prefill cost, and \citet{hu2024epic} formalize position-independent caching primitives that make such reuse practical. Subsequent work introduces lightweight query-driven scoring passes that decide which cached entries to refresh, via future-attention forecasts \citep{wang2026prophetkv} or query-centric, fusion-aware scoring \citep{yan2026qcfusequerycentriccachefusion}. These methods establish selective cache fusion for text contexts, but leave open how to transfer query-aware ranking to visual tokens and allocate one refresh budget across images.

\subsection{Multimodal KV and Visual Tokens}
\label{sec:related:mm}

Several recent systems compress or reuse the multimodal KV cache directly. \citet{zhao2025mpic} store per-image KVs on disk for cross-request reuse, and a related branch applies modality-aware or hybrid budgets \citep{wan2024lookm,pei2024csp,wan2025meda,zeng2026hybridkv}. A parallel branch studies inter-image attention itself: multi-image benchmarks expose the compositional demands of this setting \citep{zhao2024mirb,meng2024mmiu}, \citet{srikrishnan2025blindsight} find most layers carry little useful cross-image attention and apply template-aware sparse masks, and \citet{liu2026vica} and \citet{bohle2025casa} bypass self-attention over visual tokens through dedicated cross-attention pathways. A third branch reduces redundancy within a single image, first through vision-transformer token sparsification and reorganization \citep{rao2021dynamicvit,liang2022evit}, then through VLM-specific visual token pruning at prefill \citep{chen2024fastv,zhang2024sparsevlm,liu2024mustdrop} and patch fusion or instance-adaptive control \citep{cao2023pumer,cao2024madtp,guo2025crop,zhang2025adaptinfer,ma2026desap,zhao2026attentiondebiasing}. Together, these lines address multimodal cache compression, cross-image attention, and visual-token redundancy, but leave open their combination in stale-cache refresh: norm-aware token selection under one query-conditioned budget shared across images.

\subsection{Attention and the Residual Stream}
\label{sec:related:mech}

A separate line studies attention as a vector-valued contribution rather than a scalar importance weight. Attention-flow analyses show that raw weights alone do not determine information propagation \citep{abnar2020quantifying}. Most directly, \citet{kobayashi2020attention} measure a token's contribution by norm-weighted attention, $\alpha_t\|f(x_t)\|_2$, where $f$ includes the value and attention-output transformations. The residual-stream view provides the complementary systems interpretation: successive sublayers read from and write to a shared channel \citep{elhage2021transformer}, feed-forward layers write predictive features into it \citep{geva2021kv,geva2022promote}, and cross-modal information travels through the same channel in VLMs \citep{zhang2025crossmodalflow}. Mechanistic accounts further attribute attention sinks to no-op heads \citep{bondarenko2023quantizable} or fixed-coordinate massive activations \citep{sun2024massive,yang2025twosides}, explaining why high attention need not imply a large write. These analyses motivate separating how strongly the query reads a token from the scale at which that token can write. \textsc{Conduit} builds on this norm-weighted attention perspective, using an accessible pre-output cached-value-norm proxy for budgeted stale-cache refresh and pairing it with empirical cross-image relevance amplification and one global selector for both $K=1$ and $K>1$ inputs.

\section{Method}
\label{sec:method}

We frame visual KV-cache reuse as budgeted stale-cache refresh. Given cached visual K/V tensors and the current query context, we choose a small subset of visual tokens to recompute under the new prompt while reusing the rest---unselected tokens remain as potentially stale states rather than being evicted. \textsc{Conduit} ranks each state by how strongly the current query reads it and how much pre-output write scale it carries. It combines cached-key query attention with a practical pre-output value-norm proxy, then applies empirical image-level relevance amplification before one global top-$k$ selection (Figure~\ref{fig:refresh-mask}). The same rule produces the refresh mask for both single-image ($K=1$) and multi-image ($K>1$) inputs.

\begin{figure}[!t]
    \centering
    \includegraphics[width=\linewidth]{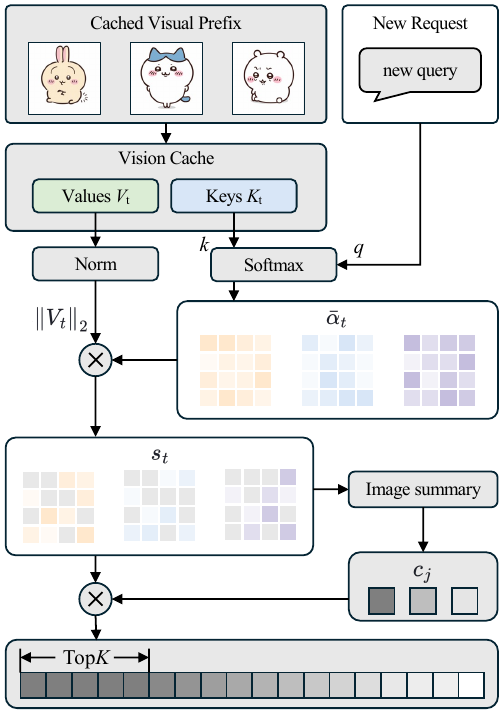}
    \caption{\textbf{\textsc{Conduit} refresh-mask construction.} Cached values provide a pre-output value-norm proxy, while cached keys and the new query estimate query attention. Their product gives each token's throughput score; a per-image coefficient reweights scores before a global top-$k$ selector chooses tokens to recompute.}
    \label{fig:refresh-mask}
\end{figure}

\subsection{Problem Setup}
\label{sec:method:setup}

Let a request contain $K$ images. Image $j$ contributes a visual-token set $\mathcal{V}_j$ with $n_j=|\mathcal{V}_j|$, and
\begin{equation}
    \mathcal{V} = \bigcup_{j=1}^{K}\mathcal{V}_j,
    \quad
    N_{\mathrm{vis}} = \sum_{j=1}^{K}n_j .
    \label{eq:visual-tokens}
\end{equation}
We write $\pi(t)$ for the image index of visual token $t$. Given a refresh ratio $r$, the visual-token refresh budget is
\begin{equation}
    k \;=\; \lfloor r N_{\mathrm{vis}}\rfloor .
    \label{eq:budget}
\end{equation}
The budget is defined only over visual tokens; new or edited text is processed normally, while the refresh mask determines which cached visual K/V entries are recomputed for the current request.

\subsection{Local Residual-Staleness Model}
\label{sec:method:residual}

We motivate the ranking rule with a local first-order model. Consider one decoder layer and one attention head at a query position $q$. Let $K_t^{\star}$ and $K_t^c$ denote the fresh and cached keys, and let $\widetilde V_t^{\star}$ and $\widetilde V_t^c$ denote the corresponding effective per-token value writes after the attention output map. Using the cached-key query attention $\alpha^c_{q,t}$ as the local readout signal from the current query to token $t$, the value-write component of the residual error after refreshing a set $\mathcal{T}\subseteq\mathcal{V}$ is
\begin{equation}
    \Delta h_q^{\mathrm{vis}}(\mathcal{T})
    \approx
    \sum_{t\notin\mathcal{T}} \alpha^c_{q,t}\,\xi_t ,
    \quad \xi_t = \widetilde V_t^{\star} - \widetilde V_t^c .
    \label{eq:residual-update}
\end{equation}
The full first-order difference decomposes as $h_q^{\star}-h_q^c=\sum_t\alpha^c_{q,t}\Delta\widetilde V_t+\sum_t\Delta\alpha_{q,t}\widetilde V_t^c+O(\Delta\alpha\,\Delta\widetilde V)$. The first term is the stale value write in Equation~\ref{eq:residual-update}; the second captures attention drift induced by fresh keys and by query-state changes propagated from earlier layers. Because evaluating attention drift and the higher-order interaction requires a fresh execution, the refresh ranking uses cached-key attention to model the leading value-write term.

To turn the vector error in Equation~\ref{eq:residual-update} into a token-wise ranking, we use the diagonal surrogate detailed in Appendix~\ref{app:surrogate}. We assume that stale displacements from different visual tokens are not consistently aligned and therefore drop their cross-token inner-product terms. After aggregating over a scoring query span $\mathcal{Q}$, the resulting square-root diagonal-risk ranking score for leaving token $t$ stale is
\begin{equation}
    o_t = \bar\alpha_t\,\|\xi_t\|_2 ,
    \quad
    \bar\alpha_t = \big(\mathbb{E}_{q\in\mathcal{Q}}[(\alpha^c_{q,t})^2]\big)^{1/2} ,
    \label{eq:oracle-score}
\end{equation}
where $\bar\alpha_t$ aggregates the cached-key attention of token $t$ over the scoring query span. The two factors of $o_t$ capture how strongly the current query reads token $t$ and how large the stale value displacement is. We cannot use $o_t$ directly, since observing $\xi_t$ would require first refreshing the token whose staleness it measures.

\subsection{Token-Level Throughput Score}
\label{sec:method:framework}

The product form builds on the norm-weighted-attention analysis of \citet{kobayashi2020attention}, which measures the transformed attention contribution including the attention output map. In cache refresh, the fresh write displacement $\|\xi_t\|_2$ is unavailable at selection time. \textsc{Conduit} replaces it with the accessible cached pre-output norm $\|V_t^c\|_2$ as a \emph{pre-output value-norm proxy}---a practical prior for relative write scale---giving the local surrogate score
\begin{equation}
    s_t^{\mathrm{sur}} = \bar\alpha_t\,\|V_t^c\|_2 .
    \label{eq:throughput-score}
\end{equation}
Under the relative-ranking assumption that larger cached value channels indicate larger potential stale writes, this surrogate favors tokens that are both strongly read and large in pre-output write scale. Our contribution is to operationalize this prior signal for budgeted stale-entry recomputation and extend the selection across images.

The local motivation above uses a single layer and head. In the deployed policy, we use a practical aggregation over the scoring query span, heads, and decoder layers (Figure~\ref{fig:refresh-mask}). For a scored layer $\ell\in\mathcal{L}$ and $H$ attention heads, the attention estimate is
\begin{equation}
    a_t^{(\ell)}
    =
    \frac{1}{|\mathcal{Q}|H}
    \sum_{q\in\mathcal{Q}}
    \sum_{h=1}^{H}\alpha^c_{q,t,\ell,h} ,
    \label{eq:head-mean}
\end{equation}
and the multi-layer throughput score is
\begin{equation}
    s_t
    =
    \frac{1}{|\mathcal{L}|}
    \sum_{\ell\in\mathcal{L}}
    a_t^{(\ell)}\,\|V_t^{c,(\ell)}\|_2 ,
    \label{eq:vnorm-multilayer}
\end{equation}
where $V_t^{c,(\ell)}$ is the concatenated multi-head cached value vector at layer $\ell$. By default $\mathcal{L}$ contains all decoder layers, so the score averages the query-conditioned attention--value signal across the full stack rather than relying on a hand-picked layer. Equation~\ref{eq:oracle-score} uses RMS attention from the squared-error surrogate, while Equation~\ref{eq:head-mean} uses the query- and head-mean available to the deployed policy. This practical aggregation is evaluated end to end.

\subsection{Image-Level Reweighting}
\label{sec:method:inter}

Token-level throughput ranks candidates within and across images, but multi-image prompts introduce competition among visually salient yet query-irrelevant images. Isolated token-level attention peaks can be noisy, whereas image-mean attention provides a coarser query-conditioned image readout. We use this summary for empirical image-level relevance amplification of the norm-aware token score. For image $j$, define
\begin{equation}
\begin{aligned}
    g_j &=
    \frac{1}{|\mathcal{L}|\,n_j}
    \sum_{\ell\in\mathcal{L}}
    \sum_{t\in\mathcal{V}_j}
    a_t^{(\ell)},
    \\
    \bar g &=
    \frac{1}{K}\sum_{j=1}^{K}g_j ,
\end{aligned}
    \label{eq:image-mean}
\end{equation}
the per-image mean cached-key attention and its cross-image mean. The image coefficient is
\begin{equation}
    c_j
    =
    (1-\lambda)
    +
    \lambda\,\frac{g_j}{\bar g},
    \quad
    \lambda\in[0,1],
    \label{eq:reweight-coef}
\end{equation}
    which interpolates between uniform scoring ($\lambda=0$) and attention-proportional amplification ($\lambda=1$). The coefficient uses $g_j$ because value norm is already present in the token score. Multiplying by $c_j$ raises tokens from images with above-average query readout in the shared global ranking while preserving an adaptive, quota-free allocation. The factor ablation and sensitivity results in Section~\ref{sec:exp:ablation} support this empirical design; Appendix~\ref{app:lagrangian-cj} gives its normalization and diagnostic behavior.

\begin{table*}[!t]
    \centering
    \small
    \setlength{\tabcolsep}{6pt}
    \renewcommand{\arraystretch}{1.15}
    \begin{tabular}{l l c c c c c @{\hskip 1em} c}
        \toprule
        & \textbf{Method} & \textbf{LongDocURL} & \textbf{MMLongBench-Doc} &\textbf{SlideVQA} & \textbf{InfoSeek} & \textbf{ViQuAE} & \textbf{Avg.} \\
        \midrule
        \multirow{7}{*}{\rotatebox[origin=c]{90}{Qwen2.5-VL-3B}}
          & Full prefill & 56.39\stdpm{0.42} & 34.19\stdpm{0.44} & 70.46\stdpm{0.42} & 44.36\stdpm{0.42} & 42.58\stdpm{0.44} & 49.60 \\
          & Cache reuse  & 47.93\stdpm{0.35} & 31.80\stdpm{0.45} & 64.59\stdpm{0.44} & 32.34\stdpm{0.40} & 34.88\stdpm{0.35} & 42.31 \\
        \cmidrule(l){2-8}
          & MPIC         & 48.16\stdpm{0.38} & 31.35\stdpm{0.40} & 64.31\stdpm{0.36} & 42.16\stdpm{0.40} & 38.00\stdpm{0.40} & 44.80 \\
          & CacheBlend   & 48.02\stdpm{0.42} & 31.18\stdpm{0.39} & 64.26\stdpm{0.39} & 32.60\stdpm{0.36} & 37.38\stdpm{0.40} & 42.69 \\
          & KVShare      & 48.10\stdpm{0.41} & 31.51\stdpm{0.37} & 64.43\stdpm{0.40} & 32.36\stdpm{0.43} & 36.33\stdpm{0.37} & 42.55 \\
          & ProphetKV    & 54.35\stdpm{0.40} & 32.25\stdpm{0.37} & 68.52\stdpm{0.44} & 42.30\stdpm{0.40} & 40.36\stdpm{0.38} & 47.56 \\
          & \textbf{\textsc{Conduit}} & \textbf{56.00}\stdpm{0.44} & \textbf{34.40}\stdpm{0.37} & \textbf{70.88}\stdpm{0.38} & \textbf{43.39}\stdpm{0.37} & \textbf{42.00}\stdpm{0.36} & \textbf{49.33} \\
        \midrule
        \multirow{7}{*}{\rotatebox[origin=c]{90}{Qwen2.5-VL-7B}}
          & Full prefill & 65.40\stdpm{0.39} & 53.43\stdpm{0.44} & 73.70\stdpm{0.39} & 54.54\stdpm{0.42} & 55.58\stdpm{0.41} & 60.53 \\
          & Cache reuse  & 56.59\stdpm{0.42} & 51.48\stdpm{0.40} & 68.83\stdpm{0.37} & 32.46\stdpm{0.39} & 45.76\stdpm{0.41} & 51.02 \\
        \cmidrule(l){2-8}
          & MPIC         & 55.93\stdpm{0.40} & 51.69\stdpm{0.39} & 68.82\stdpm{0.36} & 49.07\stdpm{0.39} & 50.73\stdpm{0.35} & 55.25 \\
          & CacheBlend   & 56.01\stdpm{0.38} & 51.23\stdpm{0.42} & 68.74\stdpm{0.36} & 32.32\stdpm{0.41} & 45.92\stdpm{0.38} & 50.84 \\
          & KVShare      & 55.35\stdpm{0.36} & 51.46\stdpm{0.43} & 67.94\stdpm{0.41} & 32.74\stdpm{0.41} & 45.78\stdpm{0.42} & 50.65 \\
          & ProphetKV    & 62.23\stdpm{0.44} & 50.91\stdpm{0.42} & 71.67\stdpm{0.41} & 47.75\stdpm{0.43} & 52.93\stdpm{0.42} & 57.10 \\
          & \textbf{\textsc{Conduit}} & \textbf{63.72}\stdpm{0.41} & \textbf{52.82}\stdpm{0.45} & \textbf{73.15}\stdpm{0.39} & \textbf{49.36}\stdpm{0.39} & \textbf{54.47}\stdpm{0.42} & \textbf{58.70} \\
        \midrule
        \multirow{7}{*}{\rotatebox[origin=c]{90}{InternVL3-9B}}
          & Full prefill & 57.37\stdpm{0.42} & 43.28\stdpm{0.40} & 72.89\stdpm{0.39} & 51.11\stdpm{0.37} & 60.17\stdpm{0.42} & 56.96 \\
          & Cache reuse  & 48.01\stdpm{0.43} & 41.02\stdpm{0.41} & 65.01\stdpm{0.39} & 38.09\stdpm{0.44} & 55.06\stdpm{0.45} & 49.44 \\
        \cmidrule(l){2-8}
          & MPIC         & 47.98\stdpm{0.35} & 42.02\stdpm{0.44} & 66.88\stdpm{0.44} & 43.79\stdpm{0.36} & 56.44\stdpm{0.42} & 51.42 \\
          & CacheBlend   & 47.87\stdpm{0.37} & 42.94\stdpm{0.36} & 66.96\stdpm{0.41} & 39.86\stdpm{0.43} & 52.20\stdpm{0.37} & 49.97 \\
          & KVShare      & 48.76\stdpm{0.37} & 40.36\stdpm{0.42} & 65.11\stdpm{0.42} & 38.04\stdpm{0.43} & 55.12\stdpm{0.43} & 49.48 \\
          & ProphetKV    & 54.93\stdpm{0.39} & 42.22\stdpm{0.43} & 71.76\stdpm{0.37} & 43.94\stdpm{0.45} & 58.83\stdpm{0.41} & 54.34 \\
          & \textbf{\textsc{Conduit}} & \textbf{56.24}\stdpm{0.40} & \textbf{43.93}\stdpm{0.44} & \textbf{72.78}\stdpm{0.40} & \textbf{46.73}\stdpm{0.36} & \textbf{60.18}\stdpm{0.44} & \textbf{55.97} \\
        \bottomrule
    \end{tabular}
    \caption{\textbf{Main shifted-prefix results at $r=0.10$.} Entries are computed with each benchmark's official metric on three VLM backbones; higher is better. \emph{Full prefill} and \emph{cache reuse} are reference anchors, and \textbf{bold} marks the best budgeted method per column (highest mean). Cells report mean \textpm{} standard deviation across $5$ random seeds (pp); the \emph{Avg.}\ column reports the unweighted mean across the five datasets. Complete budget sweeps are in Appendix~\ref{app:complete-multi}.}
    \label{tab:main-multi}
\end{table*}

\subsection{Composite Refresh Policy}
\label{sec:method:pipeline}

The final composite score for visual token $t$ is
\begin{equation}
    \hat{s}_t
    =
    s_t\,c_{\pi(t)} ,
    \label{eq:topr}
\end{equation}
and the recomputed visual-token set is selected by a single global top-$k$ rule,
\begin{equation}
    \mathcal{T}_{\mathrm{recompute}}
    =
    \operatorname{TopK}_{t\in\mathcal{V}}(\hat{s}_t,k).
    \label{eq:topr-set}
\end{equation}
The binary refresh mask marks tokens in $\mathcal{T}_{\mathrm{recompute}}$ for recomputation and reuses all other cached visual K/V entries. For any fixed non-negative composite score, global top-$k$ minimizes the induced additive ranking surrogate $\sum_{t\notin\mathcal{T}}\hat s_t^{\,2}$ under the constraint $|\mathcal{T}|=k$ (Lemma~\ref{lem:rd-topk}, Appendix~\ref{app:surrogate}). Because selection remains global, the realized per-image split adapts to each image's token-score distribution, while $c_j$ controls its pressure on the shared refresh budget.

\paragraph{Single-image limit.}
If $K=1$, then $g_1=\bar g$, so $c_1=1$ for every $\lambda$. Equation~\ref{eq:topr} reduces to $\hat s_t=s_t$, the reweighting branch in Figure~\ref{fig:refresh-mask} drops out, and \textsc{Conduit} reduces to the intra-image throughput selector. Single-image cache reuse is therefore not a separate heuristic; it is the one-image limit of the same refresh rule.

\paragraph{Implementation.}
We obtain $a_t^{(\ell)}$ and $\|V_t^{c,(\ell)}\|_2$ from a single auxiliary query-conditioned scoring pass at inference. Once these tensors are available, token scoring is linear in $N_{\mathrm{vis}}$ and the image-level step is $\mathcal{O}(K)$. All main experiments use $\lambda=1$ and all-layer score averaging, without per-benchmark tuning; Algorithm~\ref{alg:conduit} and Appendix~\ref{app:impl} give the full pseudocode, scoring query, head aggregation, and implementation controls.

\begin{figure*}[t]
    \centering
    \includegraphics[width=0.9\textwidth]{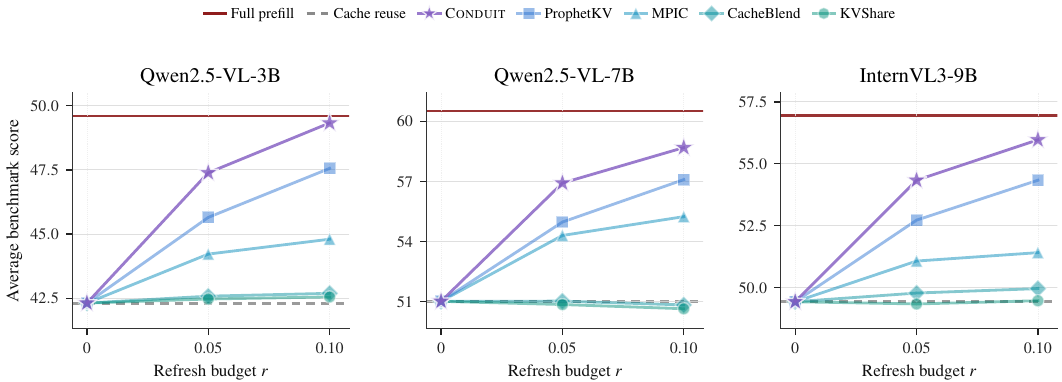}
    \caption{\textbf{Budget--quality curves by backbone.} Each panel reports the unweighted average over LongDocURL, MMLongBench-Doc, SlideVQA, InfoSeek, and ViQuAE; horizontal lines mark full prefill and zero-refresh cache reuse.}
    \label{fig:budget-quality}
\end{figure*}

\section{Experiments}
\label{sec:experiments}

\subsection{Setup}
\label{sec:exp:setup}

\paragraph{Models.}
We evaluate \textsc{Conduit} on three open VLM backbones: Qwen2.5-VL-3B, Qwen2.5-VL-7B \citep{bai2025qwen25vl}, and InternVL3-9B \citep{zhu2025internvl3}. All experiments use public checkpoints without fine-tuning.

\paragraph{Benchmarks.}
The main shifted-prefix evaluation uses five datasets: LongDocURL \citep{deng2024longdocurl}, MMLongBench-Doc \citep{ma2024mmlongbench}, SlideVQA \citep{tanaka2023slidevqa}, InfoSeek \citep{chen2023infoseek}, and ViQuAE \citep{lerner2022viquae}. To test the $K{=}1$ limit of the framework, we separately evaluate MMBench \citep{liu2024mmbench}, OCRBench \citep{liu2023ocrbench}, POPE \citep{li2023pope}, and MMStar \citep{chen2024mmstar} in their native single-image setting. Reported task scores use each benchmark's official metric.

\paragraph{Baselines.}
We compare against four recent cache-reuse and selective-recomputation methods: CacheBlend \citep{yao2024cacheblend}, ProphetKV \citep{wang2026prophetkv}, MPIC \citep{zhao2025mpic}, and KVShare \citep{yang2025kvshare}, re-evaluated in our serving setup with each benchmark's official metric. Two references anchor the comparison: \emph{full prefill}, which recomputes the entire prefix, and \emph{cache reuse}, which refreshes no visual token. Every budgeted method operates at the same refresh ratio $r$, and \textsc{Conduit} uses $\lambda=1$ with all-layer scoring and no per-benchmark tuning.

\paragraph{Cache-reuse protocol.}
For each example, we materialize a cached visual state from the recurring visual content under a cached request and then issue the benchmark-specific query under a shifted prefix that reuses the same visual content. Full prefill recomputes all visual K/V entries under the serving request; cache reuse replays the cached visual K/V without refreshing any visual token; each budgeted method refreshes exactly $\lfloor rN_{\mathrm{vis}}\rfloor$ visual entries from the same cached state. This holds cache materialization fixed across methods, so reported gaps reflect the refresh-selection rule rather than upstream cache construction.

Appendix~\ref{app:protocols} details the scoring pass, and Appendix~\ref{app:complete-multi} reports the complete results and reference anchors.

\subsection{Main Results}
\label{sec:exp:main}

Table~\ref{tab:main-multi} reports the main shifted-prefix results at $r=0.10$, and \mbox{Figure~\ref{fig:budget-quality}} plots the corresponding budget--quality curves. \textsc{Conduit} leads all budgeted methods on every dataset--backbone pair, achieves 97.0--99.5\% of the corresponding full-prefill five-dataset average, and improves over ProphetKV by about 1.7 points on average. Averaged over the three backbones, zero-refresh cache reuse loses $8.1$ points from full prefill, whereas \textsc{Conduit} loses only $1.0$ point. The advantage persists at $r=0.05$: \textsc{Conduit} has the best five-dataset average on all three backbones and ranks first in 13 of 15 dataset--backbone cells. Table~\ref{tab:single-image} shows the $K{=}1$ specialization: with $c_1=1$, \textsc{Conduit} reduces to intra-image throughput selection, and a $5\%$ refresh remains close to the full-prefill average.


\begin{table}[!t]
    \centering
    \scriptsize
    \setlength{\tabcolsep}{1pt}
    \renewcommand{\arraystretch}{1.06}
    \begin{tabular*}{\columnwidth}{@{}l@{\extracolsep{\fill}}cccccc@{}}
        \toprule
        \textbf{Method} & \textbf{$r$} & \textbf{MMBench} & \textbf{OCRBench} & \textbf{POPE} & \textbf{MMStar} & \textbf{Avg.} \\
        \midrule
        \rowcolor{conduithl}\multicolumn{7}{@{}l}{\emph{Qwen2.5-VL-3B}} \\
        Full prefill         & --   & 75.1  & 82.8  & 87.5  & 54.1  & 74.9  \\
        Cache reuse          & 0    & 74.7  & 82.4  & 87.9  & 53.4  & 74.6  \\
        \textsc{Conduit}     & 0.05 & 74.9  & 82.7  & 87.8  & 54.4  & 75.0  \\
        \midrule
        \rowcolor{conduithl}\multicolumn{7}{@{}l}{\emph{Qwen2.5-VL-7B}} \\
        Full prefill         & --   & 78.3  & 88.4  & 87.7  & 60.9  & 78.8  \\
        Cache reuse          & 0    & 77.3  & 87.9  & 86.8  & 59.7  & 77.9  \\
        \textsc{Conduit}     & 0.05 & 77.9  & 88.4  & 87.4  & 60.3  & 78.5  \\
        \midrule
        \rowcolor{conduithl}\multicolumn{7}{@{}l}{\emph{InternVL3-9B}} \\
        Full prefill         & --   & 81.0  & 84.9  & 89.7  & 58.2  & 78.4  \\
        Cache reuse          & 0    & 81.6  & 84.7  & 89.6  & 58.9  & 78.7  \\
        \textsc{Conduit}     & 0.05 & 82.7  & 84.8  & 89.6  & 58.5  & 78.9  \\
        \bottomrule
    \end{tabular*}
    \caption{\textbf{Single-image reuse in the $K{=}1$ limit.} Entries are computed with each benchmark's official metric; higher is better. Since $c_1=1$, \textsc{Conduit} reduces to intra-image throughput selection. At $r=0$ no visual token is recomputed; at $r=0.05$ the top-throughput $5\%$ are refreshed.}
    \label{tab:single-image}
\end{table}

\subsection{Serving Latency}
\label{sec:exp:latency}



\paragraph{Serving setup.}
We measure latency with a compact serving engine derived from nano-vllm and extended with visual-prefix cache materialization and refresh-mask-driven selective recomputation. The engine uses paged KV storage and fused attention kernels; all measurements use batch size 1, greedy decoding, and bf16. The scoring pass computes cached-key attention for the short scoring span with a separate matrix multiplication and softmax, leaving the fused prefill path unchanged. TTFT includes online scoring and selection but excludes one-time cache construction. The latency evaluation uses a context-stratified MMLongBench-Doc subset, so its score is separate from the aggregate in Table~\ref{tab:main-multi}.

Table~\ref{tab:ttft} reports full-path serving cost averaged over three backbones and $4$k/$8$k/$16$k contexts. At $r=0.10$, \textsc{Conduit} reduces prefill cost from 90.36 to 12.24 TFLOPs and TTFT from 0.519 to 0.174 s. The FLOP reduction is larger than the wall-clock speedup because TTFT also includes cache reads, memory traffic, kernel launches, and online scoring.

A separate 16k component profile decomposes its 0.261 s TTFT into the following rounded components: 0.020 s for cache loading, 0.068 s for scoring (26.2\%), 0.00018 s for value-norm access (0.1\%), 0.158 s for selective recomputation (60.5\%), and 0.014 s for text prefill and other work. The scoring pass is 7.3\% of the corresponding 0.93 s full-prefill TTFT. Because cached value norms are query-independent, storing one scalar per token and layer at cache construction would exactly preserve the scores while removing online value-tensor reads.

Across the nine profiled settings, \textsc{Conduit's} latency differs from ProphetKV by at most 11\,ms. On the profiled subset, its aggregate score is 1.20 points higher. KVShare has the lowest measured speedup among the budgeted methods, particularly at 16k contexts. Appendix~\ref{app:ttft-breakdown} gives the per-backbone and per-context totals.

\begin{table}[!t]
    \centering
    \footnotesize
    \setlength{\tabcolsep}{0pt}
    \renewcommand{\arraystretch}{1.06}
    \begin{tabular*}{\columnwidth}{@{}l@{\extracolsep{\fill}}cccc@{}}
        \toprule
        \textbf{Method} & \textbf{Score} & \textbf{TFLOPs} & \textbf{TTFT (s)} & \textbf{Speedup} \\
        \midrule
        Full prefill & 43.59 & 90.36 & 0.519 & 1.00$\times$ \\
        Cache reuse & 41.41 & 2.47 & 0.049 & 10.69$\times$ \\
        \cmidrule(lr){1-5}
        CacheBlend & 41.82 & 11.36 & 0.180 & 2.89$\times$ \\
        MPIC & 41.78 & 9.97 & 0.090 & 5.76$\times$ \\
        KVShare & 41.19 & 11.02 & 0.284 & 1.83$\times$ \\
        ProphetKV & 42.06 & 12.13 & 0.169 & 3.07$\times$ \\
        \textsc{Conduit} & 43.26 & 12.24 & 0.174 & 2.99$\times$ \\
        \bottomrule
    \end{tabular*}
    \caption{\textbf{Serving trade-off on a context-stratified MMLongBench-Doc latency subset.} Quality, prefill TFLOPs, and TTFT are averaged over three backbones and $4$k/$8$k/$16$k contexts. Budgeted methods use $r=0.10$; higher \emph{Score} and lower \emph{TFLOPs}/\emph{TTFT} are better.}
    \label{tab:ttft}
\end{table}


\subsection{Ablations}
\label{sec:exp:ablation}

Table~\ref{tab:lambda-sweep} evaluates the image-reweighting strength on Qwen2.5-VL-3B at $r=0.10$. The average rises from 48.53 at $\lambda=0$ to 49.33 at $\lambda=1$, while the maximum per-benchmark variation among $\lambda\in\{0.5,0.75,1.0\}$ is only 0.38 points. We use $\lambda=1$ across models and datasets without per-benchmark tuning. Setting $\lambda=0$ makes $c_j\equiv1$ and is exactly the ``drop $c_j$'' variant in Table~\ref{tab:ablations}.

\begin{table}[!tb]
    \centering
    \scriptsize
    \setlength{\tabcolsep}{1.5pt}
    \renewcommand{\arraystretch}{1.06}
    \begin{tabular*}{\columnwidth}{@{}c@{\extracolsep{\fill}}cccccc@{}}
        \toprule
        $\boldsymbol{\lambda}$ & \textbf{LongDocURL} & \textbf{MMLB-Doc} & \textbf{SlideVQA} & \textbf{InfoSeek} & \textbf{ViQuAE} & \textbf{Avg.} \\
        \midrule
        0    & 55.28 & 33.73 & 69.79 & 42.85 & 41.02 & 48.53 \\
        0.25 & 55.29 & 34.06 & 70.29 & 43.30 & 41.27 & 48.84 \\
        0.50 & 55.92 & 34.38 & \textbf{70.97} & 43.30 & 41.91 & 49.30 \\
        0.75 & 55.83 & 34.37 & 70.59 & \textbf{43.44} & 41.91 & 49.23 \\
        1.00 & \textbf{56.00} & \textbf{34.40} & 70.88 & 43.39 & \textbf{42.00} & \textbf{49.33} \\
        \bottomrule
    \end{tabular*}
    \caption{\textbf{Sensitivity to image reweighting} on Qwen2.5-VL-3B at $r=0.10$. Entries use official benchmark metrics; MMLB-Doc abbreviates MMLongBench-Doc. Higher is better. $\lambda=0$ is exactly the ``drop $c_j$'' ablation.}
    \label{tab:lambda-sweep}
\end{table}

Table~\ref{tab:ablations} isolates the two factors in the composite score. Removing the value norm turns the token score into query attention alone and causes the larger drop on two of the three backbones and on average; removing $c_j$ disables image-level amplification and gives a consistent loss. Both factors therefore contribute empirical gains to norm-aware token ranking and cross-image reweighting.

\begin{table}[!tb]
    \centering
    \footnotesize
    \setlength{\tabcolsep}{2.2pt}
    \renewcommand{\arraystretch}{1.12}
    \resizebox{\columnwidth}{!}{%
    \begin{tabular}{@{}lccc@{}}
        \toprule
        \textbf{Variant} & \textbf{Qwen2.5-VL-3B} & \textbf{Qwen2.5-VL-7B} & \textbf{InternVL3-9B} \\
        \midrule
        \textsc{Conduit} (full)        & \textbf{49.33} & \textbf{58.70} & \textbf{55.97} \\
        $\;-$ drop $c_j$               & 48.53 & 58.23 & 55.43 \\
        $\;-$ drop $\|V\|$             & 48.44 & 57.35 & 55.44 \\
        \bottomrule
    \end{tabular}%
    }
    \caption{\textbf{Factor ablation of the composite score} ($r=0.10$, averaged over the five shifted-prefix datasets). Dropping the image coefficient sets $c_j\equiv1$ (intra-image throughput only); dropping the value norm ranks by query attention alone.}
    \label{tab:ablations}
\end{table}

\subsection{Analyses}
\label{sec:exp:analyses}

Appendix~\ref{app:diagnostics} examines the composite score through surrogate-ranking, cached-norm-stability, token-set-overlap, cross-image-reweighting, and scale-similarity diagnostics. The larger-sample extensions use 64 prompts for surrogate ranking, 64 shifted-prefix pairs for cached-norm stability, and 64 images per type for scale similarity; the controlled projected-error ranking reaches Spearman $0.79$--$0.89$. These diagnostics characterize the ranking signals, while Table~\ref{tab:ablations} shows that both score factors improve end-to-end quality. On Qwen2.5-VL-3B at $r=0.10$, a manual audit of the 23 MMLongBench-Doc and LongDocURL cases in which \textsc{Conduit} and full prefill produce different answers finds 15 with limited within-image coverage of the answer-bearing page, four with near-tied evidence and distractor page scores, and four qualitatively consistent with unmodeled attention drift (Appendix~\ref{app:failure-audit}). Appendix~\ref{app:case-study} gives a representative LongDocURL example in which \textsc{Conduit} returns the same answer as full prefill.

\section{Conclusion}
\label{sec:conclusion}

We introduced \textsc{Conduit}, a training-free residual-stream restoration policy for visual KV-cache reuse under shifted prefix contexts. The method adapts norm-weighted attention to rank stale visual tokens for refresh, using cached value norm as a pre-output proxy for write scale, and adds empirical image-level relevance amplification to reweight candidates across images. With a single image, the coefficient is one and the same intra-image rule applies.

Across three open VLM backbones and nine visual benchmarks, \textsc{Conduit} requires no architectural modification and adds a single query-conditioned scoring pass. At a $10\%$ refresh budget, it achieves $97.0$--$99.5\%$ of the corresponding full-prefill five-dataset shifted-prefix average and leads the budgeted baselines on average; on a context-stratified MMLongBench-Doc latency subset, it achieves a $2.99\times$ TTFT speedup relative to full prefill. These results show that norm-aware token ranking and image-level relevance amplification provide an effective and efficient policy for shifted-prefix visual cache reuse.

\section*{Limitations}
\label{sec:limitations}

We evaluate \textsc{Conduit} on static-image and document shifted-prefix reuse; video, streaming observations, and long-horizon agentic workflows remain untested. The single-image setting offers less headroom because zero-refresh reuse is already close to full prefill. Our latency measurements use a compact nano-vllm-derived backend at batch size one and do not cover continuous-batching production servers. The method reduces repeated-prefill compute but retains the storage cost of cached visual K/V, with smaller benefits when the reusable visual prefix is short relative to the surrounding text. Finally, cached value norm is a practical ranking proxy, and the selector does not explicitly model key- or query-state-induced attention drift.

\section*{Ethics Statement}
\label{sec:ethics}

We study an inference-time scoring method for vision--language cache reuse and evaluate it solely through inference on three publicly released VLMs: Qwen2.5-VL-3B / Qwen2.5-VL-7B \citep{bai2025qwen25vl} and InternVL3-9B \citep{zhu2025internvl3}. Our experiments use publicly available benchmarks (LongDocURL, MMLongBench-Doc, SlideVQA, InfoSeek, ViQuAE, MMBench, OCRBench, POPE, and MMStar) under their respective licenses and usage terms. We train no new model, involve no human subjects, and collect no new data. \textsc{Conduit} reduces repeated-prefill FLOPs and may lower the energy cost of long-context multimodal serving, but it also lowers the operational cost of generating multimodal content at scale and therefore inherits the dual-use risks of the underlying VLMs.

\section*{Acknowledgments}

This study was supported by the Shenzhen Medical Research Fund (award no.~A2503002 to Y.L.), the National Key R\&D Program of China (grant no.~2025YFA0923500 to Y.L.), the Chinese University of Hong Kong (CUHK; award nos.~4937025, 4937026, 5501517, and 5501329), and the IdeaBooster Fund (award nos.~IDBF23ENG05 and IDBF24ENG06 to Y.L.), and was partially supported by grants from the Research Grants Council of the Hong Kong Special Administrative Region (Hong Kong SAR), China (project nos.~CUHK 24204023 and 14208525 to Y.L.), and the Innovation and Technology Commission of the Hong Kong SAR, China (project nos.~GHP/065/21SZ, ITS/247/23FP, and PRP/033/24FX to Y.L.). This research was also supported by the Research Matching Grant Scheme at CUHK (award nos.~8601603 and 8601663 to Y.L.). We also thank the Shenzhen Loop Area Institute for its support under grant FPF10120250014.

\clearpage
\bibliography{custom}

\clearpage
\appendix
\section{Implementation and Reproducibility}
\label{app:protocols}

This section records the method-specific operational details for the shifted-prefix results in Appendix~\ref{app:complete-multi}: the refresh-mask algorithm, the query-conditioned scoring pass, the fixed hyperparameters, and the native single-image evaluation protocol.

\subsection{Refresh-Mask Algorithm}
\label{app:algorithm}

Algorithm~\ref{alg:conduit} constructs the binary visual-token refresh mask used in all experiments. The scoring span $\mathcal{Q}$ is the final text segment before decoding. Selected visual positions are recomputed for the serving request; all other visual K/V entries are reused from the cache.

\subsection{Scoring Pass and Hyperparameters}
\label{app:impl}

\paragraph{Scoring forward pass.}
The scoring inputs $a_t^{(\ell)}$ and $V_t^{c,(\ell)}$ are read from a single query-conditioned forward pass over cached visual K/V. The scoring queries are the final textual block immediately before the first decode step; when an image segment is followed by an instruction or question, this block is the text after the last image segment. Cached-key attention is averaged over the scoring span and over heads, and $\|V_t^c\|_2$ is computed on the concatenated multi-head cached value vector. The single-head notation in Equation~\ref{eq:residual-update} is therefore only a mathematical simplification; the implementation follows Equation~\ref{eq:head-mean}. We use this head- and query-averaged cached-key attention $a_t$ as the deployed statistic in all experiments.

\begin{algorithm}[t]
\caption{\textsc{Conduit} refresh mask}
\label{alg:conduit}
\small
\begin{algorithmic}[1]
\Require visual-token sets $\{\mathcal{V}_j\}_{j=1}^{K}$; cached visual K/V entries; image map $\pi(t)$; query span $\mathcal{Q}$; scored layers $\mathcal{L}$; refresh ratio $r$; reweighting strength $\lambda$
\Ensure binary refresh mask $M$ over visual positions
\State $\mathcal{V} \gets \bigcup_{j=1}^{K}\mathcal{V}_j$, \quad $N_{\mathrm{vis}}\gets|\mathcal{V}|$, \quad $k\gets\lfloor rN_{\mathrm{vis}}\rfloor$
\If{$k=0$}
    \State \Return $M_t=0$ for all $t\in\mathcal{V}$
\EndIf
\State Run one query-conditioned scoring pass to read $\alpha_{q,t}^{c,(\ell,h)}$ and $V_t^{c,(\ell)}$ for $q\in\mathcal{Q}$, $t\in\mathcal{V}$, $\ell\in\mathcal{L}$, and heads $h=1,\ldots,H$
\For{$\ell\in\mathcal{L}$}
    \For{$t\in\mathcal{V}$}
        \State $a_t^{(\ell)}\gets \frac{1}{|\mathcal{Q}|H}\sum_{q\in\mathcal{Q}}\sum_{h=1}^{H}\alpha_{q,t}^{c,(\ell,h)}$
        \State $v_t^{(\ell)}\gets \|V_t^{c,(\ell)}\|_2$
    \EndFor
\EndFor
\For{$t\in\mathcal{V}$}
    \State $s_t\gets \frac{1}{|\mathcal{L}|}\sum_{\ell\in\mathcal{L}} a_t^{(\ell)}v_t^{(\ell)}$
\EndFor
\For{$j=1,\ldots,K$}
    \State $g_j\gets \frac{1}{|\mathcal{L}|\,|\mathcal{V}_j|}\sum_{\ell\in\mathcal{L}}\sum_{t\in\mathcal{V}_j}a_t^{(\ell)}$
\EndFor
\State $\bar g\gets \frac{1}{K}\sum_{j=1}^{K}g_j$
\For{$j=1,\ldots,K$}
    \State $c_j\gets (1-\lambda)+\lambda g_j/\bar g$
\EndFor
\For{$t\in\mathcal{V}$}
    \State $\hat s_t\gets s_t\,c_{\pi(t)}$
\EndFor
\State $\mathcal{T}_{\mathrm{recompute}}\gets \operatorname{TopK}_{t\in\mathcal{V}}(\hat s_t,k)$
\State \Return $M_t=\mathbf{1}[t\in\mathcal{T}_{\mathrm{recompute}}]$ for all $t\in\mathcal{V}$
\end{algorithmic}
\end{algorithm}

\paragraph{Scored layers.}
Unless otherwise stated, scores are averaged over all decoder layers. The implementation also supports single-layer and contiguous-layer-slice sweeps for diagnostics, but all main results use full-depth averaging for every backbone and benchmark.

\paragraph{Fixed hyperparameters.}
All main benchmark comparisons use $\lambda=1$ and all-layer score averaging, with no per-benchmark tuning. The setting $\lambda=1$ applies the empirical image-level relevance amplification described in Appendix~\ref{app:lagrangian-cj}, whereas $\lambda=0$ sets $c_j=1$ for every image and is exactly the ``drop $c_j$'' ablation in Table~\ref{tab:ablations}. Table~\ref{tab:lambda-sweep} reports the global sensitivity sweep used to select $\lambda=1$ once as the shared setting across models and datasets.

\subsection{Single-Image Evaluation Protocol}
\label{app:single-image}

The four single-image benchmarks use their official splits and scoring metrics. We build the cached visual state by running each image under a fixed generic instruction and storing its per-layer visual K/V tensors. At serving time, the model receives the benchmark-specific question while reusing the cached visual entry. For $r=0$, no visual token is recomputed; for $r=0.05$, the top $5\%$ visual tokens under the throughput score are refreshed using the same scored-layer schedule as the main evaluation. Since $K=1$, the image coefficient is $c_1=1$ and the protocol tests the intra-image throughput rule directly.

\section{Complete Experimental Results}
\label{app:complete-multi}

Tables~\ref{tab:complete-qwen3b}--\ref{tab:complete-intern9b} report the complete five-dataset shifted-prefix results for all backbones at $r\in\{0.05,0.10\}$. The main table shows $r=0.10$, and the text also summarizes the aggregate $r=0.05$ result. \emph{Full prefill} and \emph{cache reuse} are reference anchors.

\begin{table*}[t]
    \centering
    \small
    \setlength{\tabcolsep}{5pt}
    \renewcommand{\arraystretch}{1.12}
    \begin{tabular*}{\textwidth}{@{\extracolsep{\fill}} l c c c c c c @{}}
        \toprule
        \textbf{Method} & \textbf{LongDocURL} & \textbf{MMLongBench-Doc} & \textbf{SlideVQA} & \textbf{InfoSeek} & \textbf{ViQuAE} & \textbf{Avg.} \\
        \midrule
        Full prefill & 56.39\stdpm{0.42} & 34.19\stdpm{0.44} & 70.46\stdpm{0.42} & 44.36\stdpm{0.42} & 42.58\stdpm{0.44} & 49.60 \\
        Cache reuse & 47.93\stdpm{0.35} & 31.80\stdpm{0.45} & 64.59\stdpm{0.44} & 32.34\stdpm{0.40} & 34.88\stdpm{0.35} & 42.31 \\
        \midrule
        \multicolumn{7}{l}{\emph{$r=0.05$}} \\
        MPIC & 47.87\stdpm{0.38} & 31.66\stdpm{0.43} & 64.49\stdpm{0.41} & 40.29\stdpm{0.44} & 36.78\stdpm{0.39} & 44.22 \\
        CacheBlend & 47.61\stdpm{0.45} & 31.83\stdpm{0.44} & 64.74\stdpm{0.38} & 32.42\stdpm{0.36} & 36.30\stdpm{0.39} & 42.58 \\
        KVShare & 47.44\stdpm{0.43} & 31.89\stdpm{0.35} & 64.64\stdpm{0.36} & 32.35\stdpm{0.37} & 36.05\stdpm{0.40} & 42.47 \\
        ProphetKV & 51.62\stdpm{0.36} & 31.71\stdpm{0.36} & 66.54\stdpm{0.36} & 39.41\stdpm{0.43} & 38.98\stdpm{0.42} & 45.65 \\
        \textbf{\textsc{Conduit}} & \textbf{53.02}\stdpm{0.41} & \textbf{33.80}\stdpm{0.41} & \textbf{68.40}\stdpm{0.40} & \textbf{41.39}\stdpm{0.39} & \textbf{40.35}\stdpm{0.36} & \textbf{47.39} \\
        \midrule
        \multicolumn{7}{l}{\emph{$r=0.10$}} \\
        MPIC & 48.16\stdpm{0.38} & 31.35\stdpm{0.40} & 64.31\stdpm{0.36} & 42.16\stdpm{0.40} & 38.00\stdpm{0.40} & 44.80 \\
        CacheBlend & 48.02\stdpm{0.42} & 31.18\stdpm{0.39} & 64.26\stdpm{0.39} & 32.60\stdpm{0.36} & 37.38\stdpm{0.40} & 42.69 \\
        KVShare & 48.10\stdpm{0.41} & 31.51\stdpm{0.37} & 64.43\stdpm{0.40} & 32.36\stdpm{0.43} & 36.33\stdpm{0.37} & 42.55 \\
        ProphetKV & 54.35\stdpm{0.40} & 32.25\stdpm{0.37} & 68.52\stdpm{0.44} & 42.30\stdpm{0.40} & 40.36\stdpm{0.38} & 47.56 \\
        \textbf{\textsc{Conduit}} & \textbf{56.00}\stdpm{0.44} & \textbf{34.40}\stdpm{0.37} & \textbf{70.88}\stdpm{0.38} & \textbf{43.39}\stdpm{0.37} & \textbf{42.00}\stdpm{0.36} & \textbf{49.33} \\
        \bottomrule
    \end{tabular*}
    \caption{\textbf{Complete Qwen2.5-VL-3B shifted-prefix results at $r\in\{0.05,0.10\}$.} Entries are computed with the official metrics for LongDocURL, MMLongBench-Doc, SlideVQA, InfoSeek, and ViQuAE; higher is better. \emph{Full prefill} and \emph{cache reuse} are reference anchors, \textbf{bold} marks the best budgeted method per cell (highest mean), and \emph{Avg.}\ is the unweighted mean across the five datasets. Cells report mean \textpm{} standard deviation across $5$ random seeds (pp).}
    \label{tab:complete-qwen3b}
\end{table*}

\begin{table*}[t]
    \centering
    \small
    \setlength{\tabcolsep}{5pt}
    \renewcommand{\arraystretch}{1.12}
    \begin{tabular*}{\textwidth}{@{\extracolsep{\fill}} l c c c c c c @{}}
        \toprule
        \textbf{Method} & \textbf{LongDocURL} & \textbf{MMLongBench-Doc} & \textbf{SlideVQA} & \textbf{InfoSeek} & \textbf{ViQuAE} & \textbf{Avg.} \\
        \midrule
        Full prefill & 65.40\stdpm{0.39} & 53.43\stdpm{0.44} & 73.70\stdpm{0.39} & 54.54\stdpm{0.42} & 55.58\stdpm{0.41} & 60.53 \\
        Cache reuse & 56.59\stdpm{0.42} & 51.48\stdpm{0.40} & 68.83\stdpm{0.37} & 32.46\stdpm{0.39} & 45.76\stdpm{0.41} & 51.02 \\
        \midrule
        \multicolumn{7}{l}{\emph{$r=0.05$}} \\
        MPIC & 55.81\stdpm{0.37} & 51.81\stdpm{0.36} & 68.46\stdpm{0.38} & \textbf{46.39}\stdpm{0.43} & 49.06\stdpm{0.35} & 54.31 \\
        CacheBlend & 55.87\stdpm{0.41} & \textbf{51.86}\stdpm{0.37} & 68.85\stdpm{0.35} & 32.68\stdpm{0.36} & 45.87\stdpm{0.42} & 51.03 \\
        KVShare & 56.08\stdpm{0.44} & 51.42\stdpm{0.37} & 68.23\stdpm{0.39} & 32.81\stdpm{0.43} & 45.76\stdpm{0.44} & 50.86 \\
        ProphetKV & 61.20\stdpm{0.36} & 48.71\stdpm{0.36} & 70.38\stdpm{0.36} & 43.35\stdpm{0.38} & 51.26\stdpm{0.36} & 54.98 \\
        \textbf{\textsc{Conduit}} & \textbf{62.65}\stdpm{0.41} & 51.12\stdpm{0.39} & \textbf{71.90}\stdpm{0.37} & 46.03\stdpm{0.41} & \textbf{52.97}\stdpm{0.38} & \textbf{56.93} \\
        \midrule
        \multicolumn{7}{l}{\emph{$r=0.10$}} \\
        MPIC & 55.93\stdpm{0.40} & 51.69\stdpm{0.39} & 68.82\stdpm{0.36} & 49.07\stdpm{0.39} & 50.73\stdpm{0.35} & 55.25 \\
        CacheBlend & 56.01\stdpm{0.38} & 51.23\stdpm{0.42} & 68.74\stdpm{0.36} & 32.32\stdpm{0.41} & 45.92\stdpm{0.38} & 50.84 \\
        KVShare & 55.35\stdpm{0.36} & 51.46\stdpm{0.43} & 67.94\stdpm{0.41} & 32.74\stdpm{0.41} & 45.78\stdpm{0.42} & 50.65 \\
        ProphetKV & 62.23\stdpm{0.44} & 50.91\stdpm{0.42} & 71.67\stdpm{0.41} & 47.75\stdpm{0.43} & 52.93\stdpm{0.42} & 57.10 \\
        \textbf{\textsc{Conduit}} & \textbf{63.72}\stdpm{0.41} & \textbf{52.82}\stdpm{0.45} & \textbf{73.15}\stdpm{0.39} & \textbf{49.36}\stdpm{0.39} & \textbf{54.47}\stdpm{0.42} & \textbf{58.70} \\
        \bottomrule
    \end{tabular*}
    \caption{\textbf{Complete Qwen2.5-VL-7B shifted-prefix results at $r\in\{0.05,0.10\}$.} Entries are computed with the official metrics for LongDocURL, MMLongBench-Doc, SlideVQA, InfoSeek, and ViQuAE; higher is better. \emph{Full prefill} and \emph{cache reuse} are reference anchors, \textbf{bold} marks the best budgeted method per cell (highest mean), and \emph{Avg.}\ is the unweighted mean across the five datasets. Cells report mean \textpm{} standard deviation across $5$ random seeds (pp).}
    \label{tab:complete-qwen7b}
\end{table*}

\begin{table*}[t]
    \centering
    \small
    \setlength{\tabcolsep}{5pt}
    \renewcommand{\arraystretch}{1.12}
    \begin{tabular*}{\textwidth}{@{\extracolsep{\fill}} l c c c c c c @{}}
        \toprule
        \textbf{Method} & \textbf{LongDocURL} & \textbf{MMLongBench-Doc} & \textbf{SlideVQA} & \textbf{InfoSeek} & \textbf{ViQuAE} & \textbf{Avg.} \\
        \midrule
        Full prefill & 57.37\stdpm{0.42} & 43.28\stdpm{0.40} & 72.89\stdpm{0.39} & 51.11\stdpm{0.37} & 60.17\stdpm{0.42} & 56.96 \\
        Cache reuse & 48.01\stdpm{0.43} & 41.02\stdpm{0.41} & 65.01\stdpm{0.39} & 38.09\stdpm{0.44} & 55.06\stdpm{0.45} & 49.44 \\
        \midrule
        \multicolumn{7}{l}{\emph{$r=0.05$}} \\
        MPIC & 47.62\stdpm{0.43} & 41.77\stdpm{0.37} & 67.03\stdpm{0.40} & 42.70\stdpm{0.41} & 56.28\stdpm{0.37} & 51.08 \\
        CacheBlend & 48.27\stdpm{0.37} & 42.35\stdpm{0.40} & 66.34\stdpm{0.35} & 39.88\stdpm{0.37} & 52.12\stdpm{0.41} & 49.79 \\
        KVShare & 48.67\stdpm{0.43} & 40.50\stdpm{0.42} & 64.22\stdpm{0.42} & 38.41\stdpm{0.43} & 54.97\stdpm{0.42} & 49.35 \\
        ProphetKV & 52.55\stdpm{0.38} & 41.74\stdpm{0.41} & 69.30\stdpm{0.35} & 42.70\stdpm{0.39} & 57.36\stdpm{0.44} & 52.73 \\
        \textbf{\textsc{Conduit}} & \textbf{54.45}\stdpm{0.35} & \textbf{42.94}\stdpm{0.41} & \textbf{70.72}\stdpm{0.40} & \textbf{45.01}\stdpm{0.43} & \textbf{58.58}\stdpm{0.41} & \textbf{54.34} \\
        \midrule
        \multicolumn{7}{l}{\emph{$r=0.10$}} \\
        MPIC & 47.98\stdpm{0.35} & 42.02\stdpm{0.44} & 66.88\stdpm{0.44} & 43.79\stdpm{0.36} & 56.44\stdpm{0.42} & 51.42 \\
        CacheBlend & 47.87\stdpm{0.37} & 42.94\stdpm{0.36} & 66.96\stdpm{0.41} & 39.86\stdpm{0.43} & 52.20\stdpm{0.37} & 49.97 \\
        KVShare & 48.76\stdpm{0.37} & 40.36\stdpm{0.42} & 65.11\stdpm{0.42} & 38.04\stdpm{0.43} & 55.12\stdpm{0.43} & 49.48 \\
        ProphetKV & 54.93\stdpm{0.39} & 42.22\stdpm{0.43} & 71.76\stdpm{0.37} & 43.94\stdpm{0.45} & 58.83\stdpm{0.41} & 54.34 \\
        \textbf{\textsc{Conduit}} & \textbf{56.24}\stdpm{0.40} & \textbf{43.93}\stdpm{0.44} & \textbf{72.78}\stdpm{0.40} & \textbf{46.73}\stdpm{0.36} & \textbf{60.18}\stdpm{0.44} & \textbf{55.97} \\
        \bottomrule
    \end{tabular*}
    \caption{\textbf{Complete InternVL3-9B shifted-prefix results at $r\in\{0.05,0.10\}$.} Entries are computed with the official metrics for LongDocURL, MMLongBench-Doc, SlideVQA, InfoSeek, and ViQuAE; higher is better. \emph{Full prefill} and \emph{cache reuse} are reference anchors, \textbf{bold} marks the best budgeted method per cell (highest mean), and \emph{Avg.}\ is the unweighted mean across the five datasets. Cells report mean \textpm{} standard deviation across $5$ random seeds (pp).}
    \label{tab:complete-intern9b}
\end{table*}

\paragraph{Cache-reuse anchor.}
\label{app:sink-only}
The \emph{cache reuse} row in Table~\ref{tab:main-multi} and Tables~\ref{tab:complete-qwen3b}--\ref{tab:complete-intern9b} is the zero-refresh anchor. It reuses the stored visual K/V state verbatim under the same cached-prefix materialization used by the selective methods and runs only the text-side serving pass for the new request. The row therefore measures the fastest reuse point and the quality lost when no visual residual channel is restored.

\subsection{Budget--Quality Curves}
\label{app:budget-quality}

Figure~\ref{fig:budget-quality} traces the five-dataset average against the refresh budget $r$ on each backbone. All budgeted methods start from the zero-refresh cache-reuse anchor at $r=0$. In average score, \textsc{Conduit} lies above every budgeted baseline at both $r=0.05$ and $r=0.10$ on all three backbones, with ProphetKV the closest; CacheBlend and KVShare stay near the cache-reuse anchor across the sweep. For \textsc{Conduit}, most of the gap to full prefill closes by $r=0.05$, and the step from $r=0.05$ to $r=0.10$ adds a smaller increment. We report the main results at $r=0.10$, a low-budget operating point that recovers $97.0$--$99.5\%$ of full-prefill quality.

\paragraph{Cross-scale $\lambda$ spot check.}
The main sensitivity sweep in Table~\ref{tab:lambda-sweep} uses Qwen2.5-VL-3B. On Qwen2.5-VL-7B, an additional spot check on LongDocURL and InfoSeek gives scores of $63.08/48.85$, $63.33/49.11$, and $63.72/49.36$ at $\lambda\in\{0,0.5,1\}$, respectively. The improvement from $\lambda=0$ to $1$ therefore also appears at the larger scale, while the intermediate setting remains close to both endpoints.

\subsection{Serving Latency by Backbone and Context}
\label{app:ttft-breakdown}

Table~\ref{tab:ttft-breakdown} reports the per-backbone, per-context TTFT on the context-stratified latency subset behind Table~\ref{tab:ttft}; averaging its nine settings reproduces the TTFT column there. TTFT grows with context length for every method; full-prefill cost grows fastest, so the relative latency savings generally widen as the prefix lengthens. \textsc{Conduit} stays within 0--11\,ms of ProphetKV across the nine settings, both far below full prefill and slightly above the zero-refresh cache-reuse floor; cache reuse and MPIC are the cheapest, while KVShare is the slowest budgeted method at the longest contexts.

\paragraph{Online-component decomposition.}
Table~\ref{tab:ttft-components} reports a separate backbone-averaged profile of the \textsc{Conduit} path, decomposed into cache loading, the query-conditioned scoring pass, cached value-norm access, selective recomputation, and the remaining text-prefill cost. At 16k tokens, scoring takes 68\,ms, or $26.2\%$ of this profiled TTFT and $7.3\%$ of the corresponding 0.93\,s full-prefill profile. Selective recomputation is the largest component. Because $\|V_t^c\|_2$ depends only on cached content, an exact implementation can store one scalar per token and layer at cache construction and avoid online value-tensor reads without changing the score. The reported measurements use the current path, in which those reads account for only $0.1\%$ of the 16k profile.

\begin{table*}[t]
    \centering
    \footnotesize
    \setlength{\tabcolsep}{4pt}
    \renewcommand{\arraystretch}{1.12}
    \begin{tabular*}{\textwidth}{@{\extracolsep{\fill}} l ccc ccc ccc @{}}
        \toprule
        & \multicolumn{3}{c}{\textbf{Qwen2.5-VL-3B}} & \multicolumn{3}{c}{\textbf{Qwen2.5-VL-7B}} & \multicolumn{3}{c}{\textbf{InternVL3-9B}} \\
        \cmidrule(lr){2-4}\cmidrule(lr){5-7}\cmidrule(lr){8-10}
        \textbf{Method} & 4k & 8k & 16k & 4k & 8k & 16k & 4k & 8k & 16k \\
        \midrule
        Full prefill & 0.145 & 0.333 & 0.780 & 0.260 & 0.642 & 1.380 & 0.182 & 0.327 & 0.624 \\
        Cache reuse  & 0.040 & 0.042 & 0.052 & 0.031 & 0.035 & 0.054 & 0.052 & 0.057 & 0.074 \\
        \cmidrule(lr){1-10}
        CacheBlend & 0.092 & 0.137 & 0.273 & 0.095 & 0.161 & 0.349 & 0.113 & 0.149 & 0.250 \\
        MPIC       & 0.037 & 0.054 & 0.096 & 0.049 & 0.084 & 0.184 & 0.071 & 0.096 & 0.140 \\
        KVShare    & 0.116 & 0.184 & 0.421 & 0.124 & 0.317 & 0.653 & 0.128 & 0.189 & 0.427 \\
        ProphetKV  & 0.078 & 0.117 & 0.236 & 0.084 & 0.166 & 0.370 & 0.115 & 0.146 & 0.212 \\
        \textbf{\textsc{Conduit}} & 0.079 & 0.122 & 0.247 & 0.086 & 0.169 & 0.378 & 0.115 & 0.148 & 0.218 \\
        \bottomrule
    \end{tabular*}
    \caption{\textbf{Time to first token (s) by backbone and context length} on the context-stratified MMLongBench-Doc latency subset; lower is better. Budgeted methods (CacheBlend, MPIC, KVShare, ProphetKV, \textsc{Conduit}) use a $10\%$ refresh budget; \emph{full prefill} and \emph{cache reuse} are reference anchors. Averaging across the nine backbone--context settings yields the TTFT column of Table~\ref{tab:ttft}.}
    \label{tab:ttft-breakdown}

    \vspace{0.6em}
    \centering
    \footnotesize
    \setlength{\tabcolsep}{4pt}
    \renewcommand{\arraystretch}{1.12}
    \begin{tabular*}{\textwidth}{@{\extracolsep{\fill}} l ccc @{}}
        \toprule
        \textbf{Component} & \textbf{4k} & \textbf{8k} & \textbf{16k} \\
        \midrule
        Cache load & 6.4\% & 5.7\% & 7.7\% \\
        Scoring pass & 46.8\% & 36.4\% & 26.2\% \\
        Value-norm access & 0.1\% & 0.1\% & 0.1\% \\
        Selective recompute & 40.4\% & 52.1\% & 60.5\% \\
        Text prefill + other & 6.4\% & 5.7\% & 5.5\% \\
        \midrule
        Total & 100\% & 100\% & 100\% \\
        \bottomrule
    \end{tabular*}
    \caption{\textbf{Backbone-averaged \textsc{Conduit} TTFT composition (\%)} on MMLongBench-Doc at $r=0.10$. Each column reports the component share at one context length and is normalized to 100\%; percentages are rounded to one decimal. The corresponding total TTFTs are 0.094, 0.141, and 0.261\,s for 4k, 8k, and 16k, respectively. This profile is independent of the nine-setting aggregation in Table~\ref{tab:ttft-breakdown}; TTFT includes online scoring and selection and excludes one-time cache construction.}
    \label{tab:ttft-components}
\end{table*}

\section{Surrogate Analysis}
\label{app:theory}

This section develops the analytical token-ranking surrogate and the empirical image-level coefficient used in Section~\ref{sec:method}. Section~\ref{app:surrogate} motivates the throughput score from a local staleness model and connects global top-$k$ selection to its additive surrogate. Section~\ref{app:lagrangian-cj} gives the normalization, limiting cases, and empirical support for image-level relevance amplification.

\subsection{From Local Error to a Ranking Surrogate}
\label{app:surrogate}

The exact residual error is unavailable at scoring time. We motivate the local surrogate score $s_t^{\mathrm{sur}}=\bar\alpha_t\|V_t^c\|_2$ through a diagonal staleness view of the local model. Building on the norm-weighted-attention product from prior work, the analysis gives this signal an optimization role in budgeted cache refresh.

\paragraph{First-order decomposition and local readout signal.}
At a query position $q$ in one decoder head, the difference between the full-prefill output and the cached output decomposes as
\[
\begin{aligned}
    h_q^{\star}-h_q^c
    =\;& \sum_t \alpha^c_{q,t}\,\Delta \widetilde V_t
    + \sum_t \Delta\alpha_{q,t}\,\widetilde V_t^c \\
    & + O(\Delta\alpha\,\Delta \widetilde V),
\end{aligned}
\]
where $\widetilde V_t$ denotes the effective per-token write after the attention output map, $\Delta\widetilde V_t=\widetilde V_t^{\star}-\widetilde V_t^c=\xi_t$, and $\Delta\alpha_{q,t}=\alpha^{\star}_{q,t}-\alpha^c_{q,t}$ contains attention drift induced by fresh keys and by query-state changes propagated from earlier layers. Because evaluating attention drift requires a fresh execution, \textsc{Conduit} ranks tokens through the first, value-write term. Using cached-key attention $\alpha^c_{q,t}$ as the local readout signal, its squared local error decomposes as
\[
\begin{aligned}
    \mathbb{E}_q\Big\|\sum_{t\notin\mathcal{T}}\alpha^c_{q,t}\,\xi_t\Big\|_2^2
    =\;& \sum_{t\notin\mathcal{T}}\bar\alpha_t^2\|\xi_t\|_2^2 \\
    & + \sum_{t\neq s}\mathbb{E}_q\big[\alpha^c_{q,t}\alpha^c_{q,s}\langle\xi_t,\xi_s\rangle\big],
\end{aligned}
\]
with RMS cached-key attention $\bar\alpha_t:=(\mathbb{E}_q\,(\alpha^c_{q,t})^2)^{1/2}$. The deployed aggregation uses the mean over query positions and heads in Equation~\ref{eq:head-mean}, a practical statistic evaluated end to end.

\paragraph{Assumption 1 --- diagonal approximation.}
We drop the cross-token inner-product terms,
\[
    \mathbb{E}_q\big[\alpha^c_{q,t}\alpha^c_{q,s}\langle\xi_t,\xi_s\rangle\big]\approx 0,
    \quad t\neq s,
\]
under the modeling assumption that stale displacements from different visual tokens are not consistently aligned in residual space. Dropping these terms gives the diagonal oracle ranking score in Equation~\ref{eq:oracle-score},
\[
    o_t = \bar\alpha_t\,\|\xi_t\|_2 ,
\]
which is the square-root diagonal-risk ranking score that \textsc{Conduit} would use if $\xi_t$ were observable at scoring time.

\paragraph{Assumption 2 --- relative staleness proxy.}
The effective-write displacement norm $\|\xi_t\|_2$ is unavailable before deciding which tokens to refresh. The deployed scorer replaces it with the accessible pre-output cached value norm under the working surrogate assumption that
\[
    \mathbb{E}\big[\|\xi_t\|_2^2 \,\big|\, V_t^c\big] \;\propto\; \|V_t^c\|_2^2 ,
\]
where $V_t^c$ is the pre-output cached value. We use its norm as a monotone ranking proxy for the post-output write displacement across contexts. The cached-norm stability analysis and end-to-end factor ablation in Appendix~\ref{app:diagnostics} test the operational role of this proxy. Substitution into the diagonal oracle gives the throughput score and surrogate used in Section~\ref{sec:method:framework},
\[
    s_t^{\mathrm{sur}} = \bar\alpha_t\,\|V_t^c\|_2 ,
    \quad
    \mathcal{D}(\mathcal{T}) := \sum_{t\notin\mathcal{T}} \big(s_t^{\mathrm{sur}}\big)^2 .
\]

\paragraph{Top-$k$ optimality for the ranking surrogate.}
The following lemma identifies the exact top-$k$ solution for $\mathcal{D}$ and connects the ranking signal to the refresh operation.
\begin{lemma}[Within-surrogate optimality of top-$k$]
\label{lem:rd-topk}
Under $|\mathcal{T}|=k$, every minimizer of $\mathcal{D}$ above is $\mathcal{T}^{\star}=\operatorname{TopK}_{t\in\mathcal{V}}(s_t^{\mathrm{sur}},k)$, with arbitrary tie-breaking.
\end{lemma}
\begin{proof}
Write $\mathcal{D}(\mathcal{T})=\sum_{t\in\mathcal{V}}(s_t^{\mathrm{sur}})^2-\sum_{t\in\mathcal{T}}(s_t^{\mathrm{sur}})^2$. The first term is constant, so the top-$k$ set maximizes the second; ties give the same objective value.
\end{proof}
Residual-error behavior is evaluated empirically in Appendix~\ref{app:diagnostics}.

\paragraph{Connection to norm-weighted attention.}
\citet{kobayashi2020attention} measure a token's attention contribution as $\alpha_t\|f(x_t)\|_2$, where $f$ includes the value and attention-output transformations. \textsc{Conduit} adapts this prior signal to stale-cache selection through the accessible pre-output cached value norm, and Lemma~\ref{lem:rd-topk} connects fixed non-negative rankings to the global top-$k$ refresh operation.

\paragraph{Empirical validation.}
The diagnostics in Section~\ref{sec:exp:analyses} evaluate the surrogate ranking on real prompts. The controlled projected-error setting in Table~\ref{tab:rd-tightness} isolates the cached-key attention signal, while the factor ablation in Table~\ref{tab:ablations} measures the end-to-end value-norm contribution.

\subsection{Empirical Image-Level Relevance Amplification}
\label{app:lagrangian-cj}

\paragraph{Design and normalization.}
The token-level surrogate already ranks candidates globally across images. We add a separate empirical image-level relevance amplification to reduce the influence of isolated high-scoring patches from images that the current query reads weakly. Using the image-mean cached-key attention $g_j$ from Equation~\ref{eq:image-mean}, the coefficient is
\[
    c_j=(1-\lambda)+\lambda\,\frac{g_j}{\bar g}.
\]
The unweighted cross-image mean satisfies $K^{-1}\sum_j c_j=1$, which normalizes the mean coefficient across images to one. With $\lambda=0$, every image uses $c_j=1$; with $\lambda=1$, above-average $g_j$ amplifies all token scores from image $j$ and below-average $g_j$ attenuates them. Global top-$k$ still enforces the exact total budget, and no hard image quota is imposed.

\paragraph{Empirical support.}
The factor ablation in Table~\ref{tab:ablations} shows a consistent end-to-end gain from including $c_j$, and Table~\ref{tab:lambda-sweep} shows limited variation among the three higher settings $\lambda\in\{0.5,0.75,1.0\}$. Table~\ref{tab:cj-defense} further characterizes its behavior relative to explicit reweighting and unweighted Global TopK.

\paragraph{Single-image reduction.}
If $K=1$, then $g_1=\bar g$ and $c_1=1$ for every $\lambda$. The composite score in Equation~\ref{eq:topr} becomes $\hat s_t=s_t$, and the multi-image branch reduces exactly to the intra-image throughput selector.

\section{Diagnostics and Analyses}
\label{app:estimators}

This section provides the estimator details and prompt-level protocol behind the diagnostics summarized in Section~\ref{sec:exp:analyses}. The floated summaries first report projected-error rankings, token-set overlap, and cross-image reweighting. Table~\ref{tab:evidence-quartiles} gives the evidence-stratified results, while Figure~\ref{fig:case-longdocurl} and Table~\ref{tab:case-longdocurl} give the qualitative example discussed in Section~\ref{app:case-study}; the subsections below provide their protocols and interpretation.

\begin{table*}[!t]
    \centering
    \small
    \setlength{\tabcolsep}{6pt}
    \renewcommand{\arraystretch}{1.12}
    \begin{tabular}{l c c c c c @{\hskip 1.4em} c c c c c}
        \toprule
        & \multicolumn{5}{c}{Reduction vs.\ random} & \multicolumn{5}{c}{Reduction vs.\ $\|V\|$-only} \\
        \cmidrule(lr){2-6}\cmidrule(l){7-11}
        Rule & $r=.05$ & $.10$ & $.20$ & $.30$ & $.50$ & $r=.05$ & $.10$ & $.20$ & $.30$ & $.50$ \\
        \midrule
        $\mathrm{TopK}(\bar\alpha)$
            & 5.44 & 11.55 & 35.68 & 94.35 & 530.13
            & 5.40 & 10.96 & 33.75 & 83.97 & 471.39 \\
        $\mathrm{TopK}(\|V\|)$
            & 1.01 & 1.05 & 1.06 & 1.12 & 1.12
            & -- & -- & -- & -- & -- \\
        $\mathrm{TopK}(\bar\alpha\cdot\|V\|)$
            & 5.38 & 11.46 & 35.46 & 94.09 & 535.96
            & 5.34 & 10.87 & 33.53 & 83.74 & 476.57 \\
        \bottomrule
    \end{tabular}
    \caption{\textbf{Projected residual-error reduction for selection rules.} Each entry is the ratio $\mathcal{E}(\text{baseline})/\mathcal{E}(\text{rule})$ on $16$ LongDocURL prompts (Qwen2.5-VL-3B, layer 24); the baseline is random selection (left block) or $\mathrm{TopK}(\|V\|)$ (right block), and higher is better. Dashes mark self-comparison cells, not missing runs.}
    \label{tab:rd-tightness}
\end{table*}

\begin{table*}[!t]
    \centering
    \small
    \begin{minipage}[t]{0.46\textwidth}
        \centering
        \setlength{\tabcolsep}{4pt}
        \renewcommand{\arraystretch}{1.1}
        \begin{tabular}{l c c c}
            \toprule
            Layer & $J_{\alpha,K}$ & $J_{\alpha,V^{-}}$ & $J_{K,V^{-}}$ \\
            \midrule
            7 (shallow) & 0.009 & 0.039 & 0.039 \\
            14 (mid)    & 0.011 & \textbf{0.062} & 0.021 \\
            24 (deep)   & 0.009 & 0.055 & 0.016 \\
            \midrule
            Random deciles & 0.053 & 0.053 & 0.053 \\
            \bottomrule
        \end{tabular}
        \caption{\textbf{Overlap among token-index sets.} Pairwise Jaccard overlaps over $59$ image segments from Qwen2.5-VL-3B captures; each set contains the token indices selected by one scalar criterion, and $V^{-}$ is the bottom decile of value norm. Values near the random-decile baseline ($J^{\star}\approx0.053$) are consistent with independence, whereas substantially smaller values indicate anti-overlap.}
        \label{tab:kv-sink}
    \end{minipage}\hfill
    \begin{minipage}[t]{0.46\textwidth}
        \centering
        \setlength{\tabcolsep}{4pt}
        \renewcommand{\arraystretch}{1.1}
        \begin{tabular}{l c c c}
            \toprule
            Staleness & Actual & Explicit & Global \\
            \midrule
            Isotropic, $\kappa=0.1$ & 0.902 & 0.894 & \textbf{0.872} \\
            Isotropic, $\kappa=0.5$ & 0.902 & 0.893 & \textbf{0.872} \\
            Isotropic, $\kappa=1.0$ & 0.905 & 0.899 & \textbf{0.875} \\
            Two-query & 0.888 & 0.894 & \textbf{0.877} \\
            \bottomrule
        \end{tabular}
        \caption{\textbf{Cross-image reweighting diagnostic.} Values are $\mathcal{E}(\mathcal{T})/\mathcal{E}_{\mathrm{uniform}}$ at $r=0.10$ on $16$ multi-image Qwen2.5-VL-3B prompts; lower is better, and $\mathcal{E}_{\mathrm{uniform}}$ is the error of uniform random refresh at the same budget under the same sampled displacement. \emph{Actual} is the deployed composite score, \emph{Explicit} uses quotas proportional to $c_jn_j$, and \emph{Global} ranks by $s_t$ without image reweighting.}
        \label{tab:cj-defense}
    \end{minipage}
\end{table*}

\begin{table}[t]
    \centering
    \small
    \setlength{\tabcolsep}{8pt}
    \renewcommand{\arraystretch}{1.08}
    \begin{tabular}{@{}lcc@{}}
        \toprule
        \textbf{Gap quartile} & \textbf{Accuracy} & \textbf{Evidence-page budget} \\
        & \textbf{Ours / full} & \textbf{Ours / raw-attn} \\
        \midrule
        Q1 (near tie) & 46.2 / 46.7 & 26.1\% / 28.5\% \\
        Q2 & 36.5 / 39.9 & 36.1\% / 34.4\% \\
        Q3 & 42.5 / 43.3 & 49.7\% / 43.7\% \\
        Q4 & 30.3 / 32.9 & 76.9\% / 69.1\% \\
        \bottomrule
    \end{tabular}
    \caption{\textbf{Accuracy and evidence-page budget by attention-gap quartile} on the MMLongBench-Doc evidence-annotated subset (Qwen2.5-VL-3B, $r=0.10$). Metrics are computed within each quartile.}
    \label{tab:evidence-quartiles}
\end{table}

\subsection{Monte Carlo Staleness Estimator}
\label{app:mc-bias}

The cross-image diagnostic in Table~\ref{tab:cj-defense} requires a per-token displacement $\xi_t$. Since $\xi_t$ is not available when the selector is run, we use two complementary estimators.

\paragraph{Isotropic Monte Carlo.}
For each token, sample
\[
    \xi_t^{(m)}
    \sim
    \mathcal{N}
    \left(
        0,\,
        \frac{\kappa^2\|V_t\|_2^2}{d}I
    \right),
    \quad
    m=1,\ldots,M,
\]
where $d$ is the value-vector dimensionality and $\kappa$ is a diagnostic scale parameter for the value-norm staleness proxy of Section~\ref{app:surrogate}. Then $\mathbb{E}[\xi_t^{(m)}]=0$ and $\mathbb{E}[\|\xi_t^{(m)}\|_2^2]=\kappa^2\|V_t\|_2^2$. Because the reported error ratio uses a uniform-refresh error generated with the same global scale, $\kappa^2$ cancels in expectation. The three $\kappa$ rows in Table~\ref{tab:cj-defense} are repeated Monte Carlo checks, and their small differences reflect sampling variation. For a fixed refresh set, the estimator
\[
    \hat{\mathcal{E}}^{\mathrm{iso}}(\mathcal{T})
    =
    \frac{1}{M}
    \sum_{m=1}^{M}
    \left\|
        \sum_{t\notin\mathcal{T}}
        \alpha_{q,t}\xi_t^{(m)}
    \right\|_2^2
\]
is unbiased for the isotropic residual-error functional and has Monte Carlo variance $O(M^{-1})$. We use $M=5$ paired direction samples per token and compare every selection rule under the same sampled staleness model.

\paragraph{Two-query proxy.}
The two-query row uses the pre-output diagnostic displacement $\delta_t=V_t^B-V_t^A$, where $V_t^A$ and $V_t^B$ are value vectors produced by two text queries over the same visual content. The notation $\delta_t$ distinguishes this diagnostic quantity from the post-$W^O$ effective-write displacement $\xi_t$ in the local model. It measures prompt-specific pre-output value drift while holding visual content fixed.

\begin{figure*}[!t]
    \centering
    \begin{minipage}[t]{0.235\textwidth}
        \centering
        \includegraphics[width=\linewidth,height=0.16\textheight,keepaspectratio]{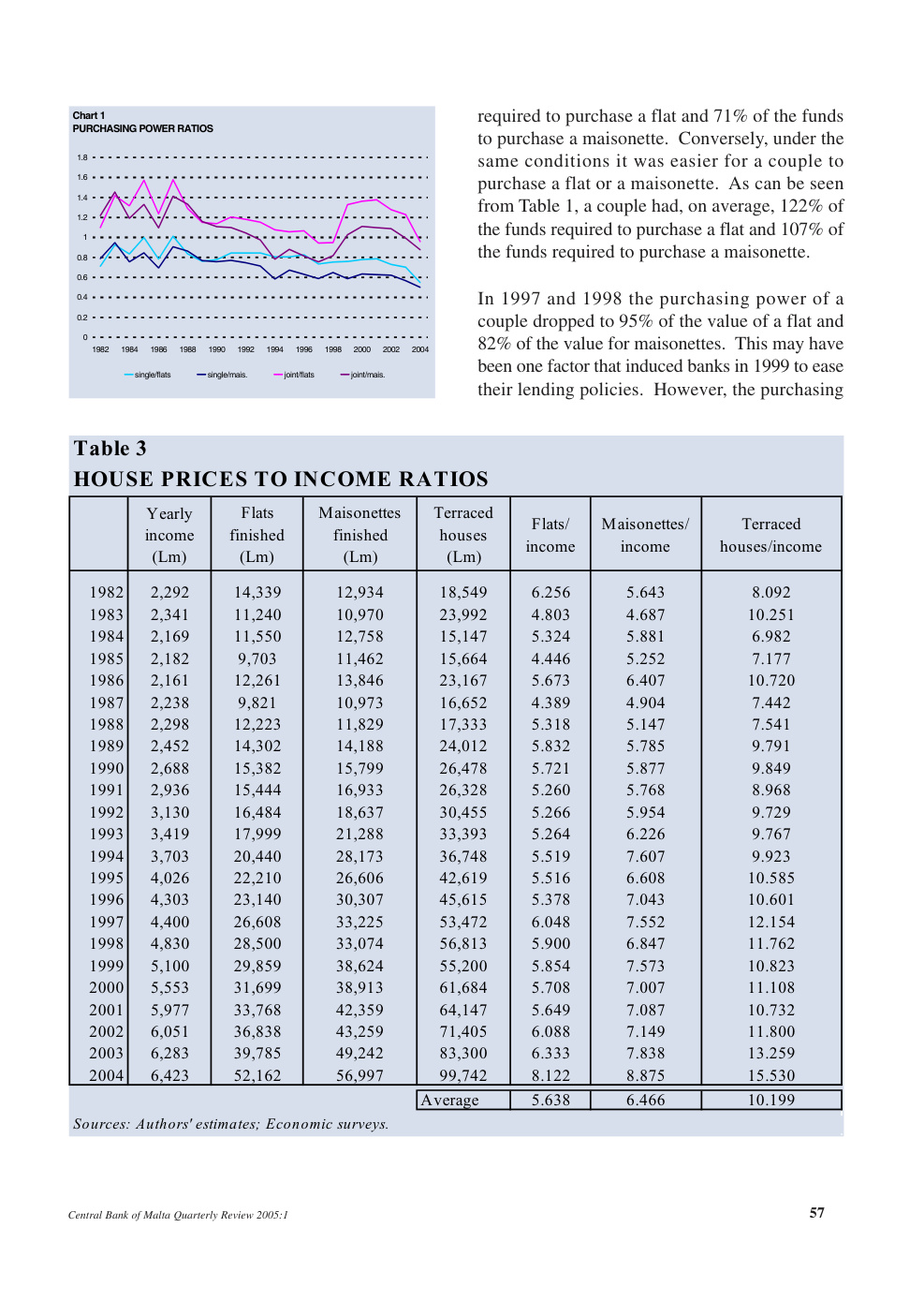}\\[-0.2em]
        {\scriptsize Context page 1}
    \end{minipage}\hfill%
    \begin{minipage}[t]{0.235\textwidth}
        \centering
        \includegraphics[width=\linewidth,height=0.16\textheight,keepaspectratio]{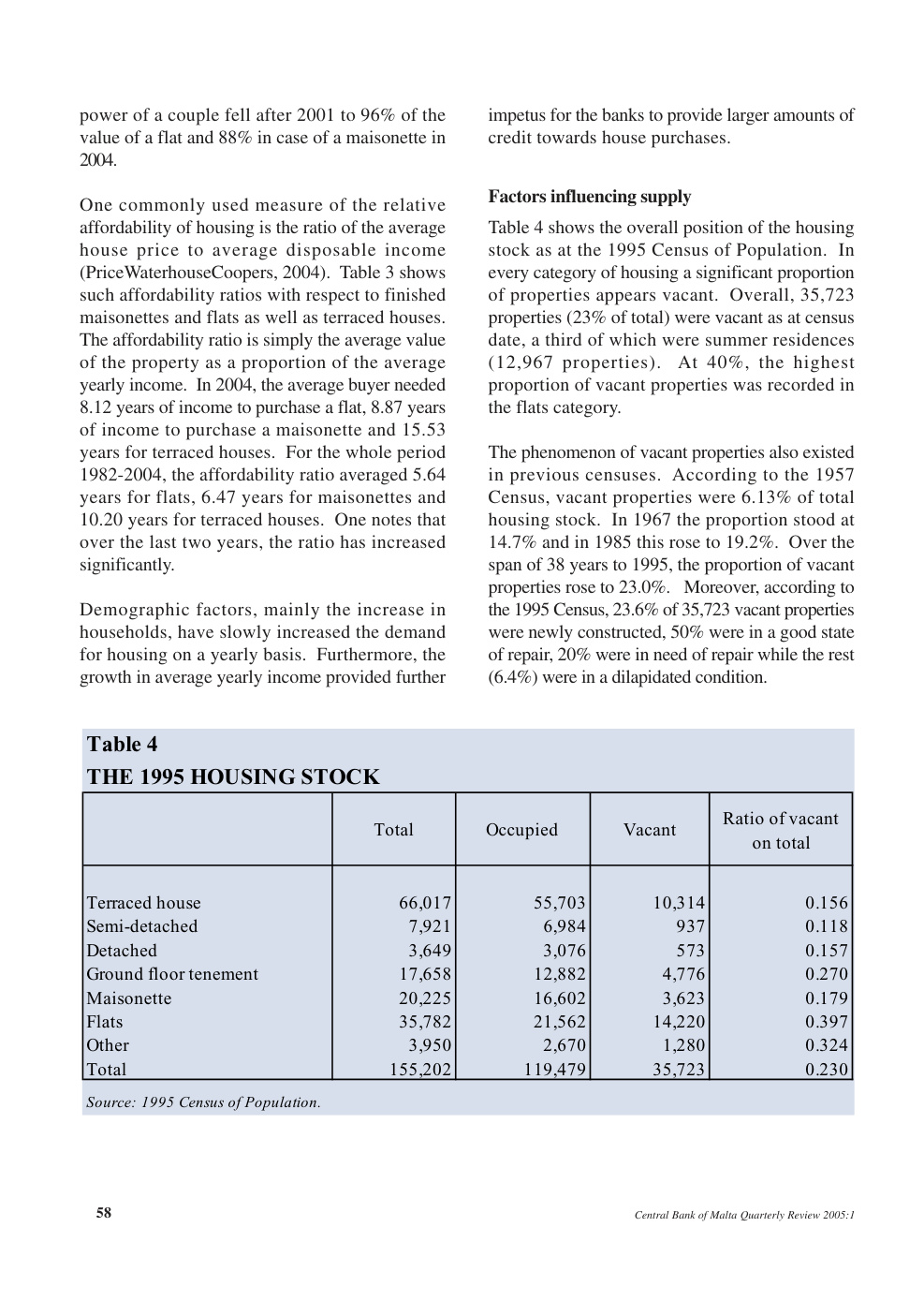}\\[-0.2em]
        {\scriptsize Context page 2}
    \end{minipage}\hfill%
    \begin{minipage}[t]{0.235\textwidth}
        \centering
        \includegraphics[width=\linewidth,height=0.16\textheight,keepaspectratio]{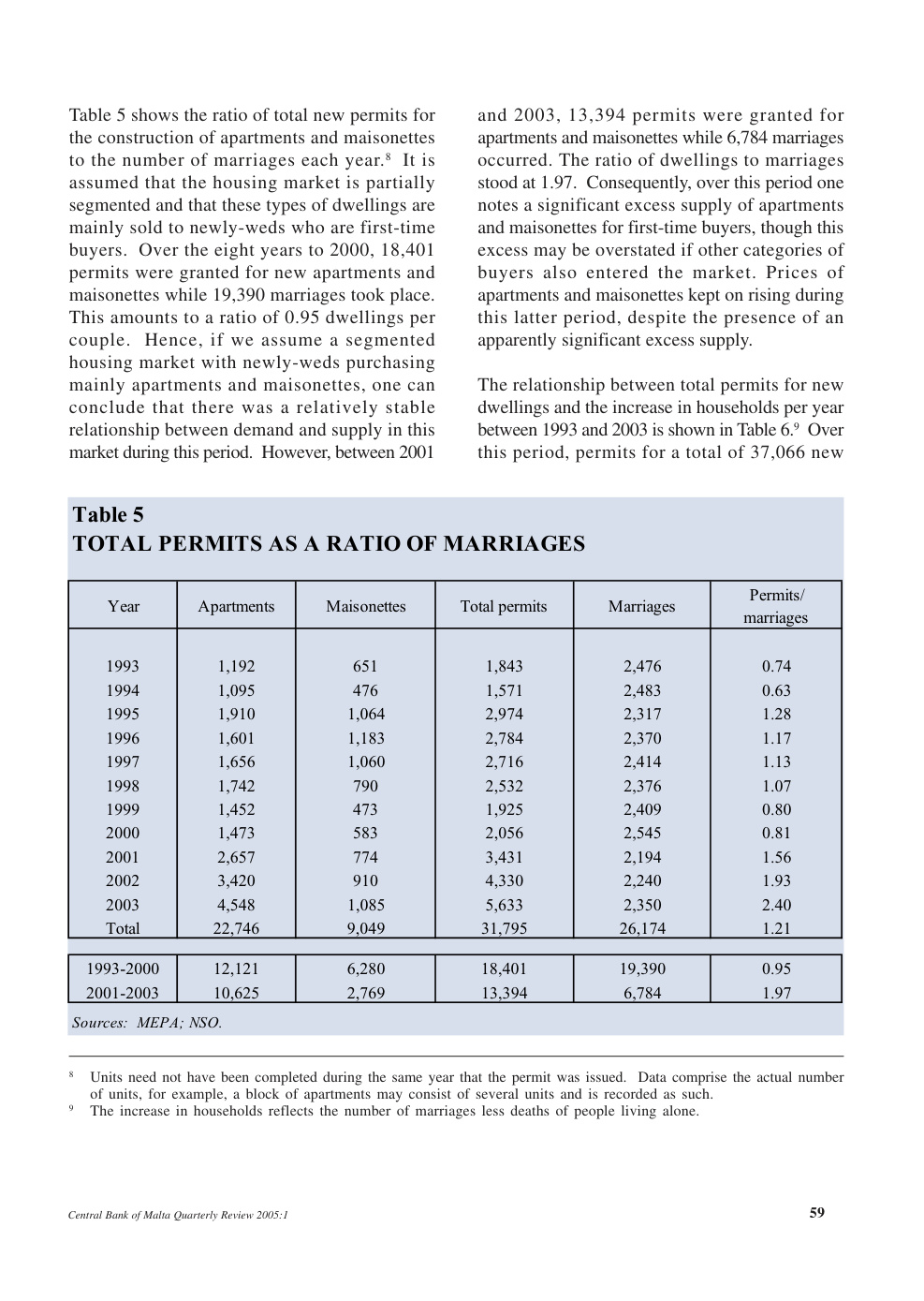}\\[-0.2em]
        {\scriptsize Context page 3}
    \end{minipage}\hfill%
    \begin{minipage}[t]{0.235\textwidth}
        \centering
        \includegraphics[width=\linewidth,height=0.16\textheight,keepaspectratio]{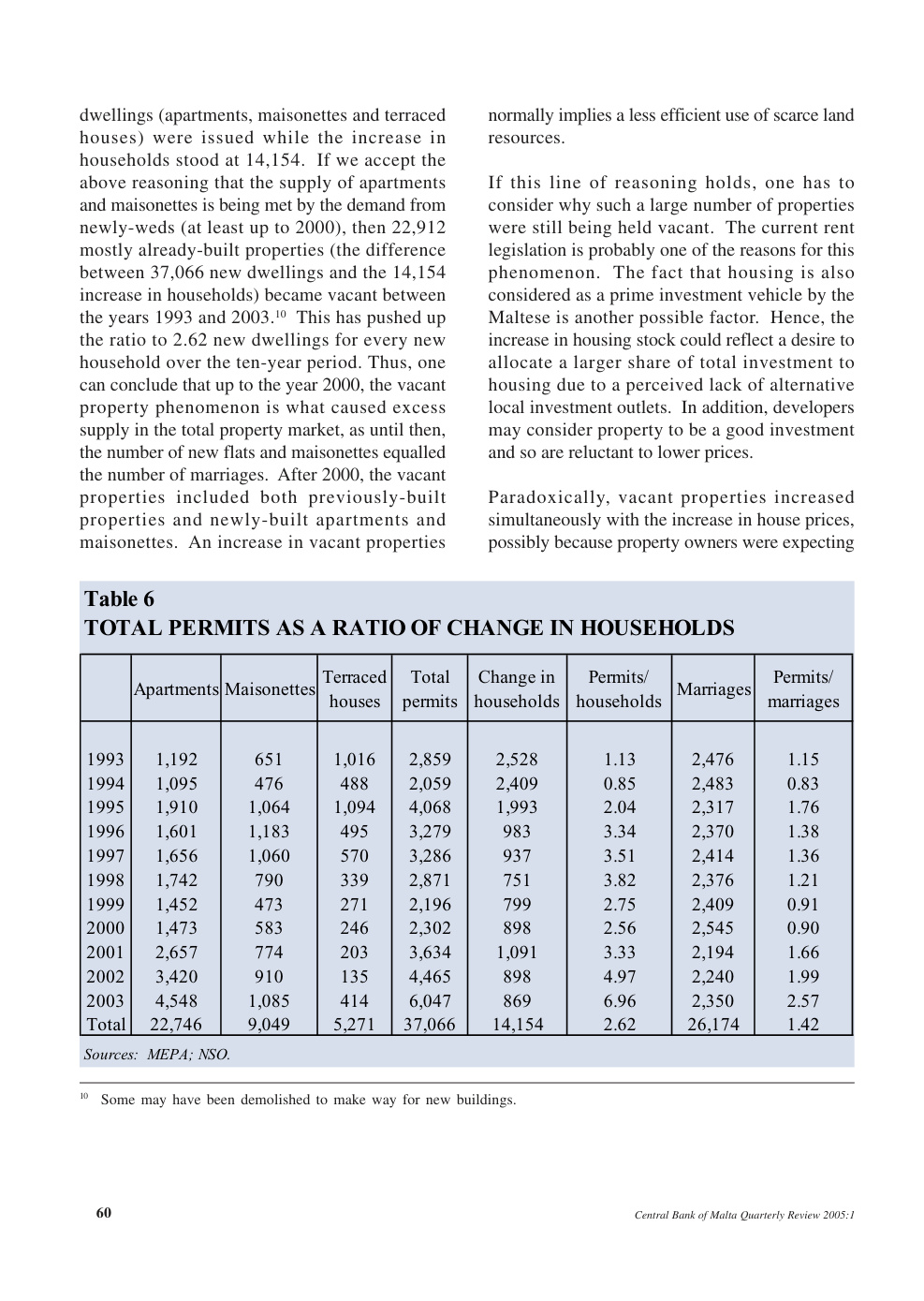}\\[-0.2em]
        {\scriptsize Context page 4}
    \end{minipage}
    \par
    \caption{\textbf{LongDocURL qualitative example.} The prompt contains four document pages and asks for the ratio associated with the period 2001--2003 and the row for marriages. Correct answering requires binding the question to the target table tuple within the document context, rather than selecting a salient nearby number.}
    \label{fig:case-longdocurl}

    \vspace{0.15em}
    \centering
    \small
    \setlength{\tabcolsep}{6pt}
    \renewcommand{\arraystretch}{1.08}
    \begin{tabular}{@{}lcp{0.64\textwidth}@{}}
        \toprule
        \textbf{Method} & \textbf{Answer} & \textbf{Qualitative behavior} \\
        \midrule
        Full prefill & 1.97 & Matches the ground-truth table entry. \\
        Cache reuse & 2.62 & Returns a distractor ratio from the same document context. \\
        CacheBlend & 2.62 & Returns the same distractor as zero-refresh reuse. \\
        MPIC & 2.62 & Returns the same distractor as zero-refresh reuse. \\
        KVShare & 2.62 & Returns the same distractor as zero-refresh reuse. \\
        ProphetKV & 0.95 & Returns a value corresponding to a different period. \\
        \textbf{\textsc{Conduit}} & \textbf{1.97} & Returns the same answer as full prefill. \\
        \bottomrule
    \end{tabular}
    \captionof{table}{\textbf{Model outputs for the LongDocURL case study.} All budgeted methods use $r=0.10$ with Qwen2.5-VL-3B-Instruct. The table reports exact generated answers and summarizes their relation to the target table entry for the same query.}
    \label{tab:case-longdocurl}
\end{figure*}

\subsection{Analysis Protocol}
\label{app:analyses}

The three diagnostics share a common capture protocol. We sample $16$ multi-image prompts from the LongDocURL development set, with a mean of $3.7$ images per prompt, and run them through Qwen2.5-VL-3B. We record per-layer attention $\alpha_{q,t}$ and value vectors $V_t$ at layers 7, 14, and 24. Tables~\ref{tab:rd-tightness} and~\ref{tab:kv-sink} use the same captures; Table~\ref{tab:cj-defense} uses the isotropic staleness estimator with $M=5$ samples per token and the two-query proxy above. The large-pool independence baseline in Table~\ref{tab:kv-sink} is the expected Jaccard overlap of two independent top-fraction-$p$ subsets of the same pool, $J^{\star}=p^2/(2p-p^2)=p/(2-p)$, giving $J^{\star}\approx0.053$ for the top decile ($p=0.1$).

Three extensions use larger samples. The surrogate-ranking correlation is measured on 64 prompts. The scale-similarity diagnostic uses 64 document pages, 64 slides, and 64 natural images, comparing the within-image squared-throughput quantile profiles before and after normalization by $g_j^2$. Cached-norm stability uses 64 LongDocURL shifted-prefix pairs and computes tokenwise Spearman correlation between value-norm rankings under the cached and serving contexts at every layer.

\subsection{Diagnostic Results}
\label{app:diagnostics}

\paragraph{Surrogate ranking.}
Table~\ref{tab:rd-tightness} compares selection rules by their reduction in projected residual error $\mathcal{E}(\mathcal{T})$ on Qwen2.5-VL-3B at layer 24. For this diagnostic, we set the stale value of every non-refreshed visual token to $V_t^c=0$ and measure the layer-output residual error at the final text-token query. Across refresh ratios $r\in\{0.05,0.10,0.20,0.30,0.50\}$, rules using query attention reduce error by roughly $5\times$ to over $500\times$ relative to random selection, while value-norm-only selection remains close to random. Extending this diagnostic to 64 prompts, the Spearman correlation between the surrogate ordering and projected residual-error ordering ranges from $0.79$ to $0.89$ across budgets. This controlled setting supports the query-attention ranking, while the end-to-end ablation in Table~\ref{tab:ablations} measures the contribution of the value-norm factor in the deployed policy.

\paragraph{Cached-norm stability.}
Across 64 LongDocURL shifted-prefix pairs, the tokenwise value-norm ranking is highly stable across contexts. The layer-averaged Spearman correlation is 0.996 for Qwen2.5-VL-3B and 0.990 for Qwen2.5-VL-7B; the minimum over layers is 0.989 and 0.979, respectively. Together with the end-to-end drops of 0.89, 1.35, and 0.53 points when the value-norm factor is removed from the three backbones, these results support $\|V_t^c\|_2$ as a stable and useful cached ranking signal in the evaluated shifted-prefix setting.

\paragraph{Token-set overlap.}
Table~\ref{tab:kv-sink} compares token-index sets induced by the top decile of RMS attention $\bar\alpha$, the top decile of key norm $\|K\|$, and the bottom decile of value norm $\|V\|$. The attention--low-value overlap is close to the independent-decile baseline in the middle and deep layers, while several other pairs fall substantially below it, indicating anti-overlap. Overall, the criteria induce distinct token rankings with no systematic positive enrichment.

\paragraph{Cross-image reweighting.}
Table~\ref{tab:cj-defense} compares the empirical coefficient with two alternatives. \emph{Conduit-actual} ranks tokens by $\hat s_t=s_t c_{\pi(t)}$; \emph{Conduit-explicit} converts the same coefficient into quotas $k_j\propto c_jn_j$ before applying in-image top-$k$; and \emph{Global TopK} ranks by $s_t$ without image reweighting. The two coefficient-based variants remain within a few percent. The comparison characterizes cross-image reweighting behavior; the end-to-end ablation in Table~\ref{tab:ablations} supplies the task-level evidence for $c_j$.

\paragraph{Scale similarity.}
Normalizing each image's squared-throughput quantile profile by $g_j^2$ reduces the mean cross-type Kolmogorov--Smirnov distance from 0.66 to 0.24 across document pages, slides, and natural images (64 images per type; normalized distances range from 0.18 to 0.26 by type). This descriptive evidence shows that normalization makes the profiles more comparable across image types.

\subsection{Evidence Allocation, Failure Audit, and Case Study}
\label{app:failure-audit}
\label{app:case-study}

Table~\ref{tab:evidence-quartiles} stratifies the evidence-annotated subset of MMLongBench-Doc by the gap between the largest raw-attention score on an evidence page and the largest score on a non-evidence page. Q1 contains the closest page-level ties; in that ambiguous stratum, \textsc{Conduit} remains within 0.5 points of full prefill and allocates 26.1\% of the budget to evidence pages, compared with 28.5\% for raw-attention top-$k$. From Q2 through Q4, \textsc{Conduit} allocates progressively more evidence-page budget than raw-attention top-$k$ as the page-level gap widens.

In the audited evaluation outputs, we examine the 23 cases on MMLongBench-Doc and LongDocURL in which \textsc{Conduit} and full prefill produce different answers for Qwen2.5-VL-3B at $r=0.10$. In 15 cases, the answer-bearing page receives nontrivial page-level budget but at most 20\% of its visual tokens are refreshed; five of these cases recover when the budget increases to $r=0.20$. Four cases have near-tied evidence and distractor page scores, limiting the separation supplied by $c_j$. The remaining four show behavior qualitatively consistent with the unmodeled attention-drift term in Section~\ref{sec:method}. The audit therefore organizes the observed discrepancies into insufficient within-image coverage, page-level ambiguity, and a smaller attention-drift-consistent category.

Figure~\ref{fig:case-longdocurl} and Table~\ref{tab:case-longdocurl} give one LongDocURL example of how selective refresh changes the generated answer in a long visual prefix. The model is Qwen2.5-VL-3B-Instruct, and all budgeted methods use the same refresh ratio $r=0.10$. The query asks for the ratio of permits for new apartments and maisonettes to the number of marriages between 2001 and 2003. Answering requires locating the relevant table entry among four document pages, rather than reading the most visually salient numeric cell. Cache reuse, CacheBlend, MPIC, and KVShare all return 2.62, a distractor ratio from the same document context, while ProphetKV returns 0.95, corresponding to a different period. \textsc{Conduit} returns 1.97 and matches full prefill in this example. The example records the observed output differences under the norm-aware, image-reweighted selection policy.

\FloatBarrier
\section{Forward-Looking Extensions}
\label{app:extensions}

The evaluated protocol covers two-request reuse of static images and document pages. Extending the restoration view to video, streaming observations, or long-horizon agents would require additional mechanisms and validation: the image index could become a frame or segment index, the coefficient could incorporate temporal drift, and the static value-norm proxy could be augmented by a learned or analytic displacement predictor. Future work will test these extensions in dynamic settings.

\end{document}